\documentclass[english,10pt,twocolumn,letterpaper]{article}

\pdfoutput=1 

\usepackage{pgfplots}
\pgfplotsset{compat=newest}
\usetikzlibrary{plotmarks}
\usepackage{grffile}

\usepackage{latexsym}
\usepackage{babel}
\usepackage{subfig}
\usepackage{amsmath,epsfig,amssymb,amsbsy, amsthm}
\usepackage{multirow}
\usepackage{nicefrac}

\usepackage[bordercolor=white,backgroundcolor=gray!30,linecolor=black,colorinlistoftodos]{todonotes}

\usepackage{wrapfig,graphicx}
\usepackage{url}

\usepackage{color}
\usepackage{colortbl}

\newcommand{\bitem}{\begin{itemize}}
\newcommand{\eitem}{\end{itemize}}

\newcommand{\bpm}{\begin{pmatrix}}      
\newcommand{\epm}{\end{pmatrix}}

\newcommand{\normc}[1]{| #1 |}
\newcommand{\bbR}{\mathbb{R}}

\providecommand{\iprod}[2]{\langle#1,#2\rangle}

\newcommand{\bi}{\begin{itemize}}
\newcommand{\ei}{\end{itemize}}

\newcommand{\Ss}{\mathcal{S}}

\newcommand{\cref}[1]{ {\tiny[{#1}]}}
\newcommand{\tm}[1]{}

\newcommand{\beq}{\begin{equation}}
\newcommand{\eeq}{\end{equation}}
\newcommand{\beqa}{\begin{eqnarray}}
\newcommand{\eeqa}{\end{eqnarray}}
\newcommand{\bc}{\begin{center}}
\newcommand{\ec}{\end{center}}

\newcommand \Sone       {{{\Ss}^1}}                           % S^1
\newcommand \TV         {{TV}}                              % TV

\newtheorem{prop}{Proposition}
\newtheorem{cor}{Corollary}

\newcommand{\dat}{\boldsymbol{\rho}}
\newcommand{\dats}{\boldsymbol{\sigma}}
\newcommand{\reg}{\boldsymbol{\Phi}}
\newcommand{\Gr}{\mathbf{1}}
\newcommand{\ul}{\boldsymbol{u}}
\newcommand{\vl}{\boldsymbol{v}}
\newcommand{\gl}{\boldsymbol{g}}
\newcommand{\ql}{\boldsymbol{q}}

\usepackage{cvpr}
\usepackage{times}
\usepackage{epsfig}
\usepackage{graphicx}
\usepackage{color}
\usepackage{tikz}
\usepackage{pgfplots}

\usepackage{scalefnt}
\usepackage{authblk}

\usetikzlibrary{spy,calc}

\newcommand{\figROF}{
\begin{figure}[t!]
  \centering
  \captionsetup[subfloat]{labelformat=empty,justification=centering,singlelinecheck=false}
  \subfloat[][\scriptsize Direct Optimization of \eqref{eq:cont_rof},\\
  $t=0.6s$, $11.78$ MB]{
    \includegraphics[width=0.148\textwidth]{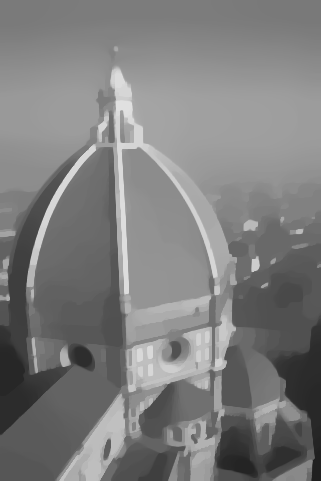}
  }
  \subfloat[][\scriptsize Baseline ($L=8$),\\ $t=\infty$, $113$ MB ]{
    \includegraphics[width=0.148\textwidth]{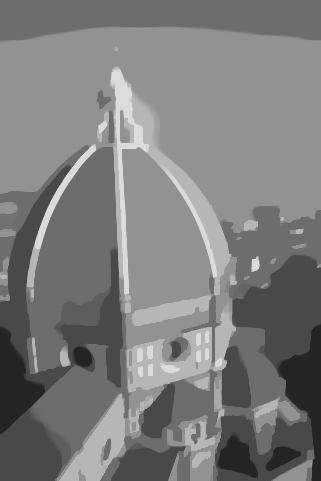}
  }
  \subfloat[][\scriptsize Baseline ($L=16$),\\ $t=\infty$, $226$ MB]{
    \includegraphics[width=0.148\textwidth]{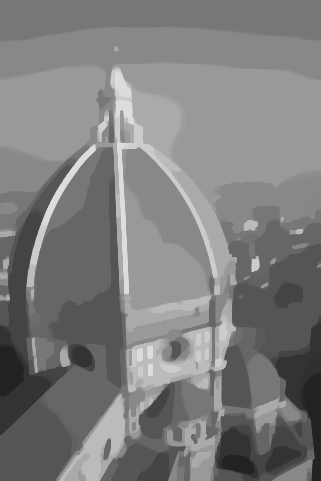}
  }\\[-0.24cm]
  \subfloat[][\scriptsize Baseline ($L=256$),\\ $t=\infty$, $3619$ MB]{
    \includegraphics[width=0.148\textwidth]{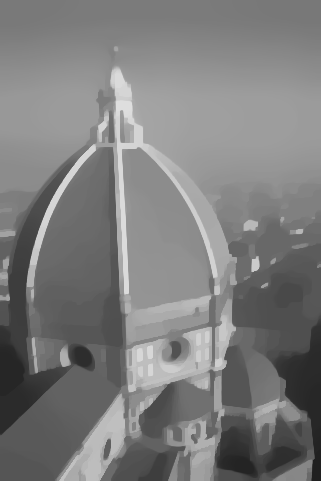}
  }
  \subfloat[][\scriptsize Proposed ($L=2$) \\ $t=1s$, $27$ MB]{
    \includegraphics[width=0.148\textwidth]{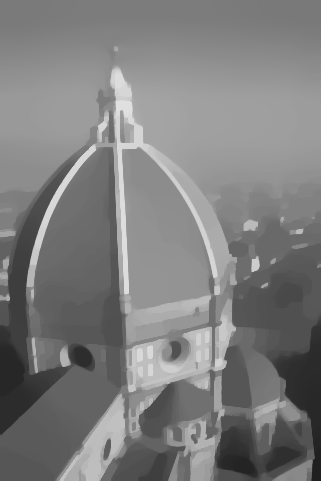}
  }
  \subfloat[][\scriptsize Proposed ($L=10$) \\ $t=15s$, $211$ MB]{
    \includegraphics[width=0.148\textwidth]{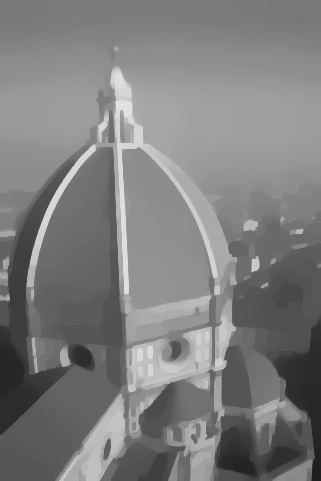}
  }\\[1mm]
  \caption{Denoising comparison. We compare the proposed method to the baseline method
    \cite{PCBC-SIIMS} on the convex ROF problem. We show the time in
    seconds required for each method to produce a solution within a
    certain energy gap to the optimal solution. As the baseline method
    optimizes a piecewise linear approximation of the quadratic
    dataterm, it fails to reach that optimality gap even for $L = 256$
    (indicated by $t=\infty$). In contrast, while the proposed lifting
    method can solve a large class of non-convex problems, it is
    almost as efficient as direct methods on convex problems.}% \vspace{-0.5cm} }
  \label{fig:rof_compare}
\end{figure}
}

\newcommand{\figTruncatedROF}{
\begin{figure}[t!]
  \centering
  \captionsetup[subfloat]{labelformat=empty,justification=centering,singlelinecheck=false}
  \subfloat[][\scriptsize Input image $f$]{
    \includegraphics[width=0.115\textwidth]{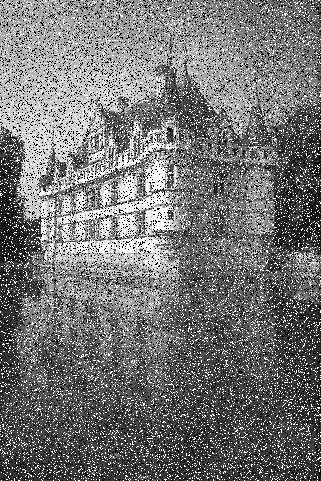}
  }
  \subfloat[][\scriptsize Proposed ($L=5$),\\$E=20494$, $t=14.6s$]{
    \includegraphics[width=0.115\textwidth]{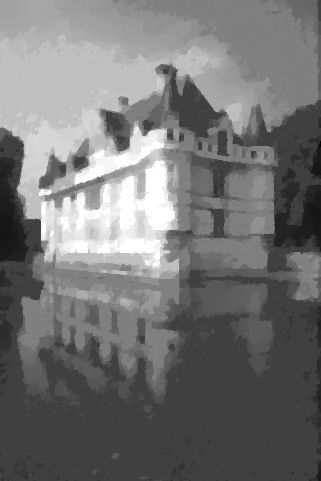}
  }
  \subfloat[][\scriptsize Proposed ($L=10$),\\$E=18844$, $t=30.5s$]{
    \includegraphics[width=0.115\textwidth]{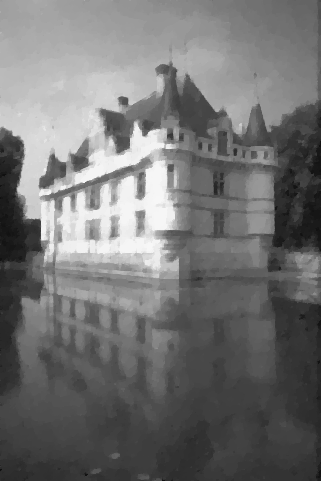}
  }
  \subfloat[][\scriptsize Proposed ($L=20$), \\$E=18699$, $t=123.9s$]{
    \includegraphics[width=0.115\textwidth]{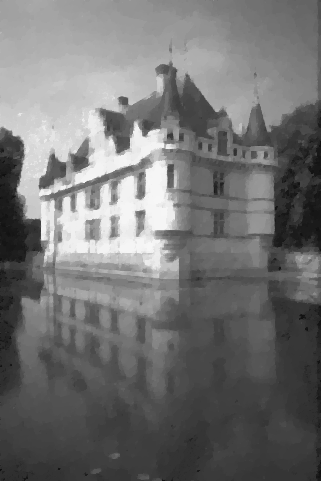}
  }\\[-0.24cm]
  \subfloat[][\scriptsize Baseline ($L=256$),\\ $E=18660$, $t=1001s$]{
    \includegraphics[width=0.115\textwidth]{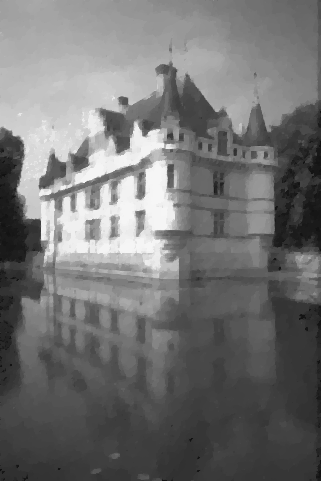}
  }
  \subfloat[][\scriptsize Baseline ($L=5$),\\ $E=23864$, $t=4.7s$]{
    \includegraphics[width=0.115\textwidth]{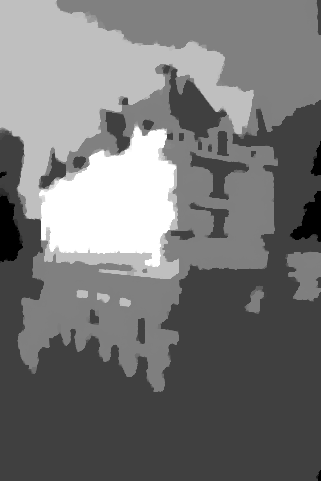}
  }
  \subfloat[][\scriptsize Baseline ($L=10$), \\ $E=19802$, $t=6.3s$]{
    \includegraphics[width=0.115\textwidth]{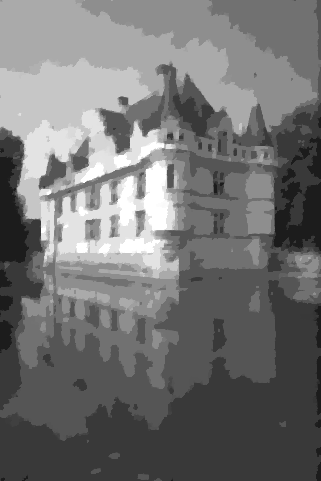}
  }
  \subfloat[][\scriptsize Baseline ($L=20$), \\ $E=18876$, $t=12.8s$]{
    \includegraphics[width=0.115\textwidth]{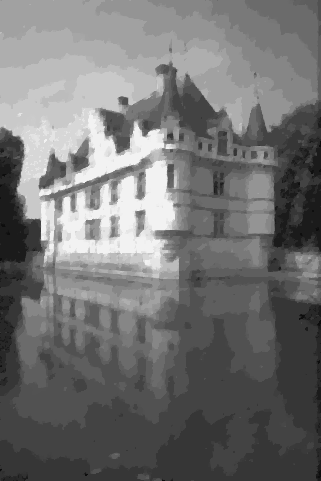}
  }\\[1mm]
  \caption{Denoising using a robust truncated quadratic dataterm. The top row shows the input image along with the result obtained by our approach for a varying number of labels $L$. The bottom row illustrates the results obtained by the baseline method \cite{PCBC-SIIMS}. The energy of the final solution as well as the total runtime are given below each image. }%With $L=10$ labels we are able to reach a lower energy than the baseline method with $L=20$.}% Without a loss in visual quality our method with $L=20$ labels runs more than 8 times faster than the baseline approach with $L=256$.}}
  \label{fig:rof_robust}
\end{figure}
}

\newcommand{\figZachKohli}{
\begin{figure}
  \captionsetup[subfloat]{labelformat=empty,justification=centering,singlelinecheck=false}
  \subfloat[][\scriptsize $E=279394$]{
    \includegraphics[width=0.111\textwidth]{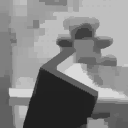} 
  }
  \subfloat[][\scriptsize $E=208432$]{
    \includegraphics[width=0.111\textwidth]{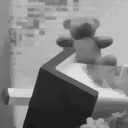} 
  }
  \subfloat[][\scriptsize $E=\textbf{196803}$]{
    \includegraphics[width=0.111\textwidth]{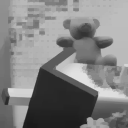} 
  }
  \subfloat[][\scriptsize $E=194855$]{
    \includegraphics[width=0.111\textwidth]{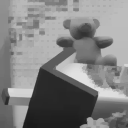} 
  }
  \\[-0.35cm]
  \subfloat[][\scriptsize $E=278108$]{
    \includegraphics[width=0.111\textwidth]{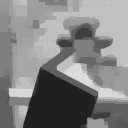} 
  }
  \subfloat[][\scriptsize $E=\textbf{208112}$]{
    \includegraphics[width=0.111\textwidth]{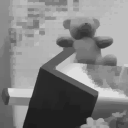} 
  }
  \subfloat[][\scriptsize $E=196810$]{
    \includegraphics[width=0.111\textwidth]{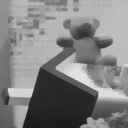} 
  }
  \subfloat[][\scriptsize $E=194845$]{
    \includegraphics[width=0.111\textwidth]{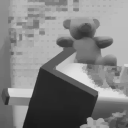} 
  }
  \\[-0.35cm]
  \subfloat[][\scriptsize $E=\textbf{277970}$]{
    \includegraphics[width=0.111\textwidth]{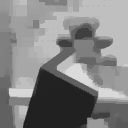} 
  }
  \subfloat[][\scriptsize $E=208493$]{
    \includegraphics[width=0.111\textwidth]{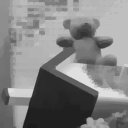} 
  }
  \subfloat[][\scriptsize $E=196979$]{
    \includegraphics[width=0.111\textwidth]{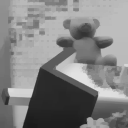} 
  }
  \subfloat[][\scriptsize $E=\textbf{194836}$]{
    \includegraphics[width=0.111\textwidth]{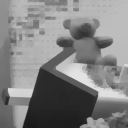} 
  }\\[0.01cm]
  \caption{Comparison to the MRF approach presented in
    \cite{Zach-Kohli-eccv12}. The first row shows DC-Linear, second
    row DC-MRF and third row our results for $4$, $8$, $16$ and $32$
    convex pieces on the truncated quadratic energy \eqref{eq:cont_rof_robust}. Below the figures we show the final nonconvex
    energy. We achieve competitive results while using a more 
    compact representation and generalizing to isotropic regularizers. \vspace{-0.4cm}}
  \label{fig:zach_compare}
\end{figure}

}

\newcommand{\figStereo}{
\begin{figure*}[t!]
  \centering
  \captionsetup[subfloat]{labelformat=empty,justification=centering,singlelinecheck=false,margin=0pt}
  \subfloat[\scriptsize One of the input images]{
    \includegraphics[width=0.156\textwidth]{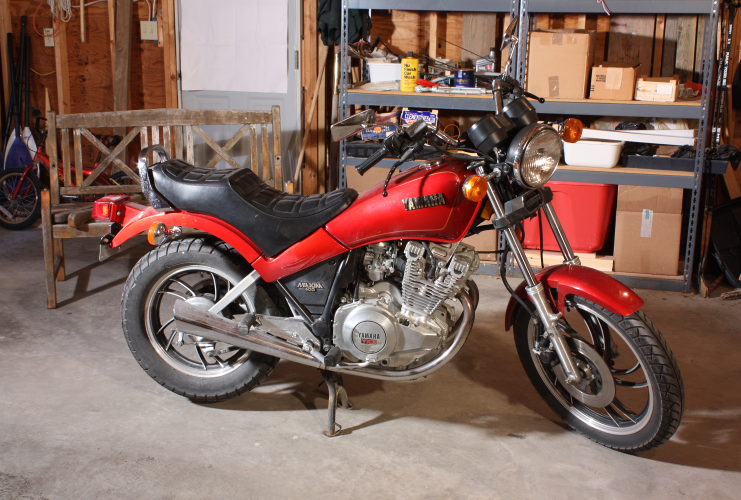}
  }
  \subfloat[\scriptsize Proposed ($L=2$)]{
    \includegraphics[width=0.156\textwidth]{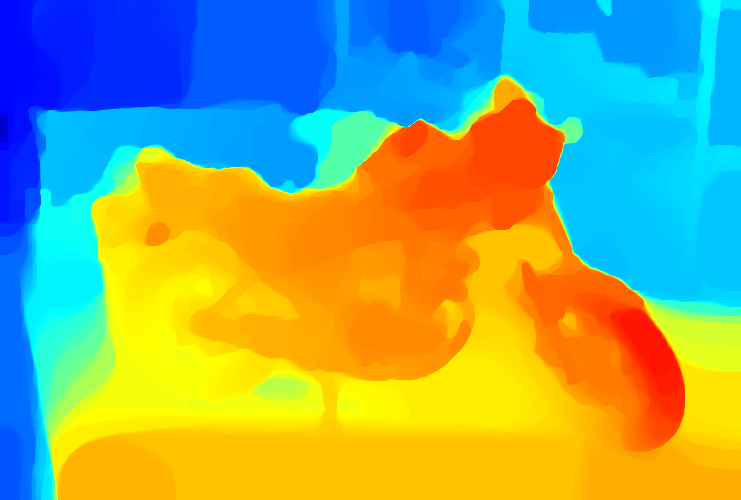}
  }
  \subfloat[\scriptsize Proposed ($L=4$)]{
    \includegraphics[width=0.156\textwidth]{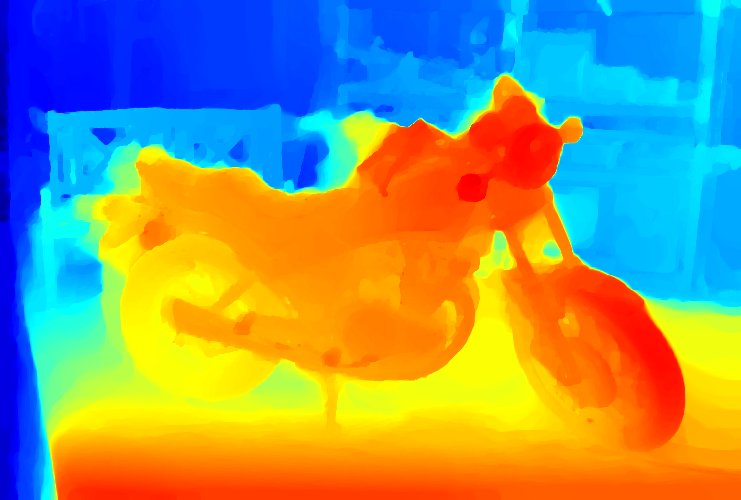}
  }
  \subfloat[\scriptsize Proposed ($L=8$)]{
    \includegraphics[width=0.156\textwidth]{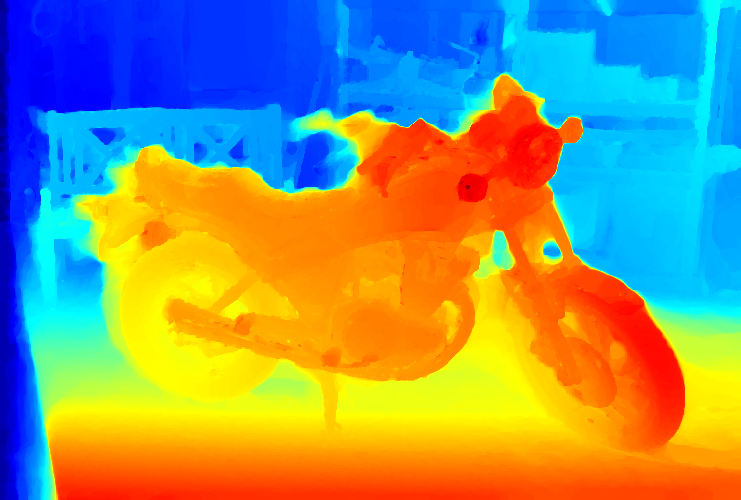}
  }
  \subfloat[\scriptsize Proposed ($L=16$)]{
    \includegraphics[width=0.156\textwidth]{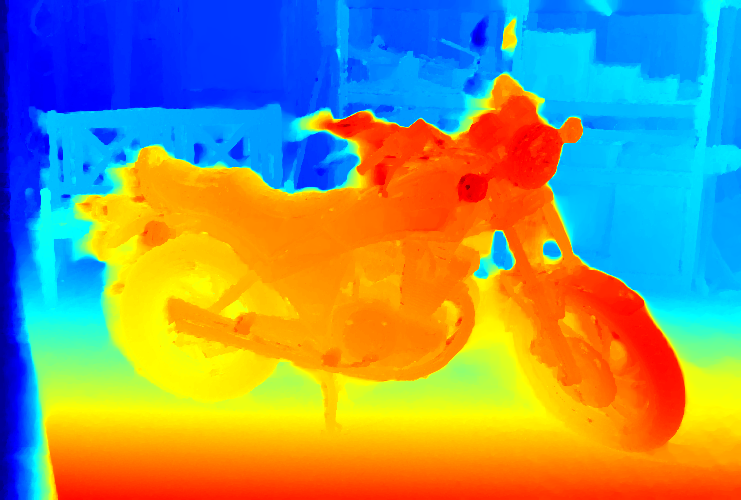}
  }
  \subfloat[\scriptsize Proposed ($L=32$)]{
    \includegraphics[width=0.156\textwidth]{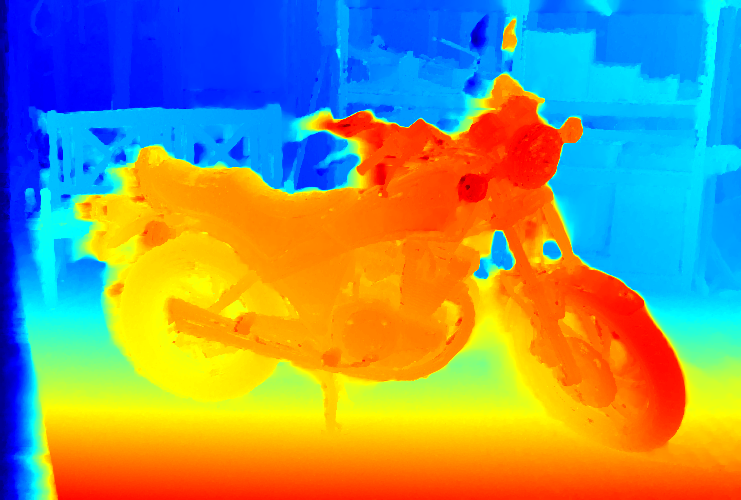}
  }\\[-0.23cm]
  \subfloat[\scriptsize Baseline ($L=270$)]{
    \includegraphics[width=0.156\textwidth]{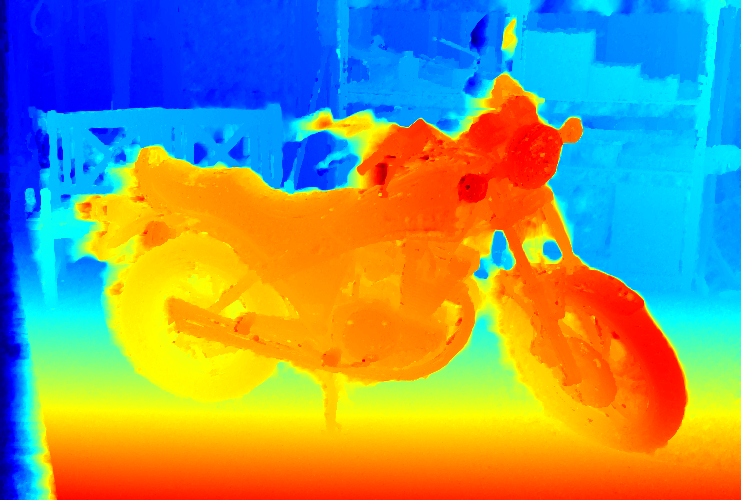}
  }
  \subfloat[\scriptsize Baseline ($L=2$)]{
    \includegraphics[width=0.156\textwidth]{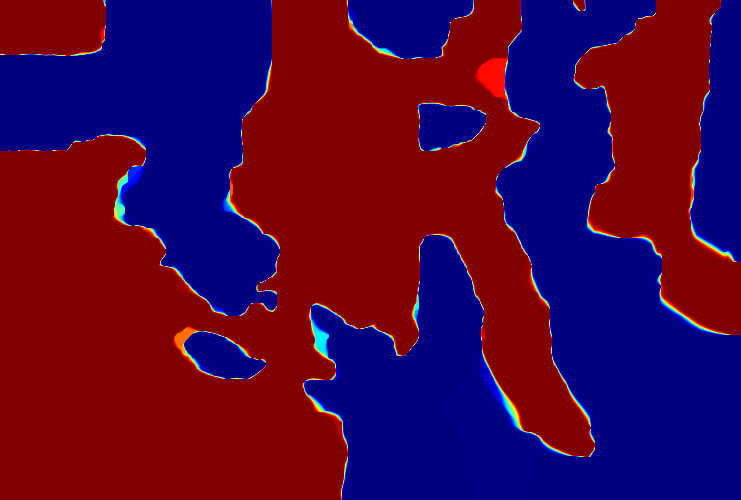}
  }
  \subfloat[\scriptsize Baseline ($L=4$)]{
    \includegraphics[width=0.156\textwidth]{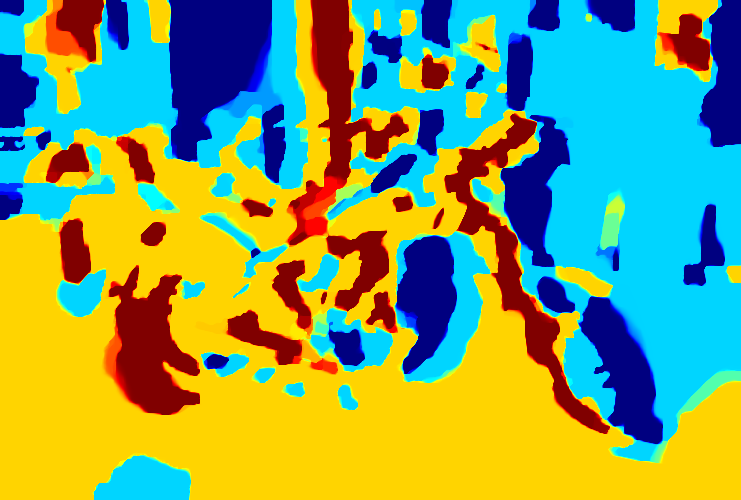}
  }
  \subfloat[\scriptsize Baseline ($L=8$)]{
    \includegraphics[width=0.156\textwidth]{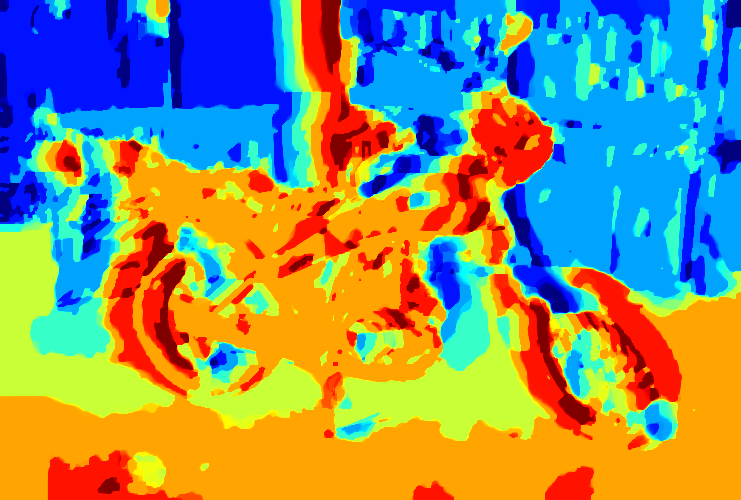}
  }
  \subfloat[\scriptsize Baseline ($L=16$)]{
    \includegraphics[width=0.156\textwidth]{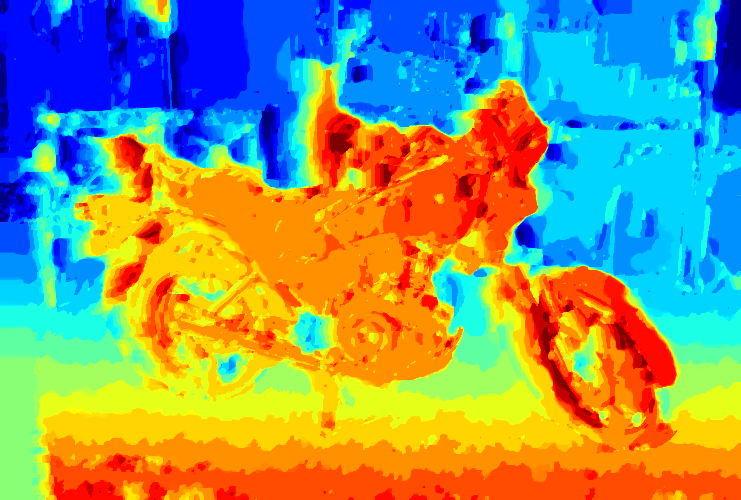}
  }
  \subfloat[\scriptsize Baseline ($L=32$)]{
    \includegraphics[width=0.156\textwidth]{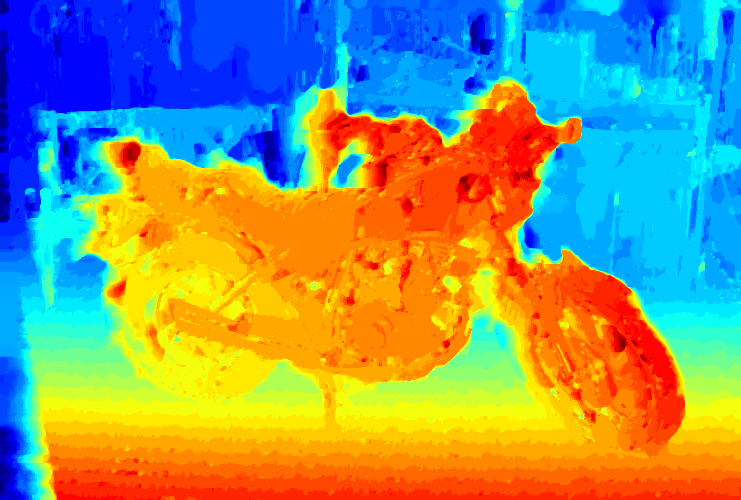}
  }\\[1mm]
  \caption{Stereo comparison. We compare the proposed method to the baseline method
     on the example of stereo matching.
    The first column shows one of the two input images and below the
    baseline method with the full number of labels. The
    proposed relaxation requires much fewer labels to reach a smooth
    depth map. Even for $L=32$, the label space discretization of the 
    baseline method is strongly visible, while the proposed method 
    yields a smooth result already for $L=8$.}
  \label{fig:stereo_compare}
\end{figure*}
}

\newcommand{\figAniso}{
\begin{figure}[t!]
  \centering
  \subfloat[][Anisotropic Regularization]{
    \includegraphics[width=0.23\textwidth]{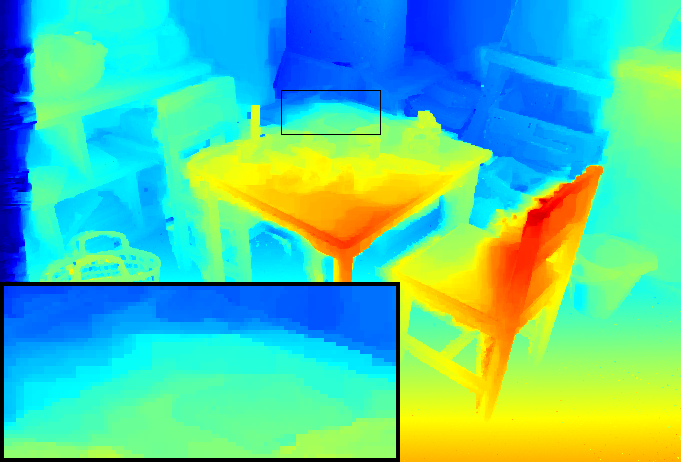} 
    \label{fig:aniso}
  }
  \subfloat[][Isotropic Regularization]{
    \includegraphics[width=0.23\textwidth]{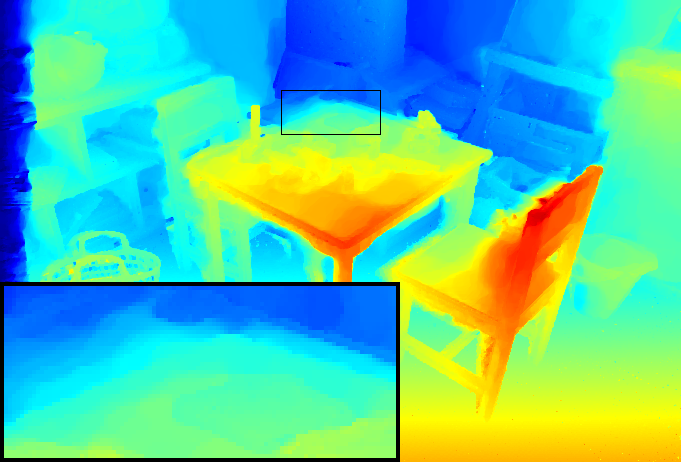} 
    \label{fig:iso}
  }\\[1mm]
  \caption{
    We compare the proposed relaxation
    with anistropic regularizer to isotropic
    regularization on the stereo matching example. Using an 
    anisotropic formulation as in~\cite{Zach-Kohli-eccv12} leads to grid bias.}
\end{figure}
}

\newcommand{\figDFF}{
\begin{figure*}[t!]
  \centering
  \captionsetup[subfloat]{labelformat=empty,justification=centering,singlelinecheck=false,margin=0pt}
  \subfloat[\scriptsize One of the input images]{
    \includegraphics[width=0.156\textwidth]{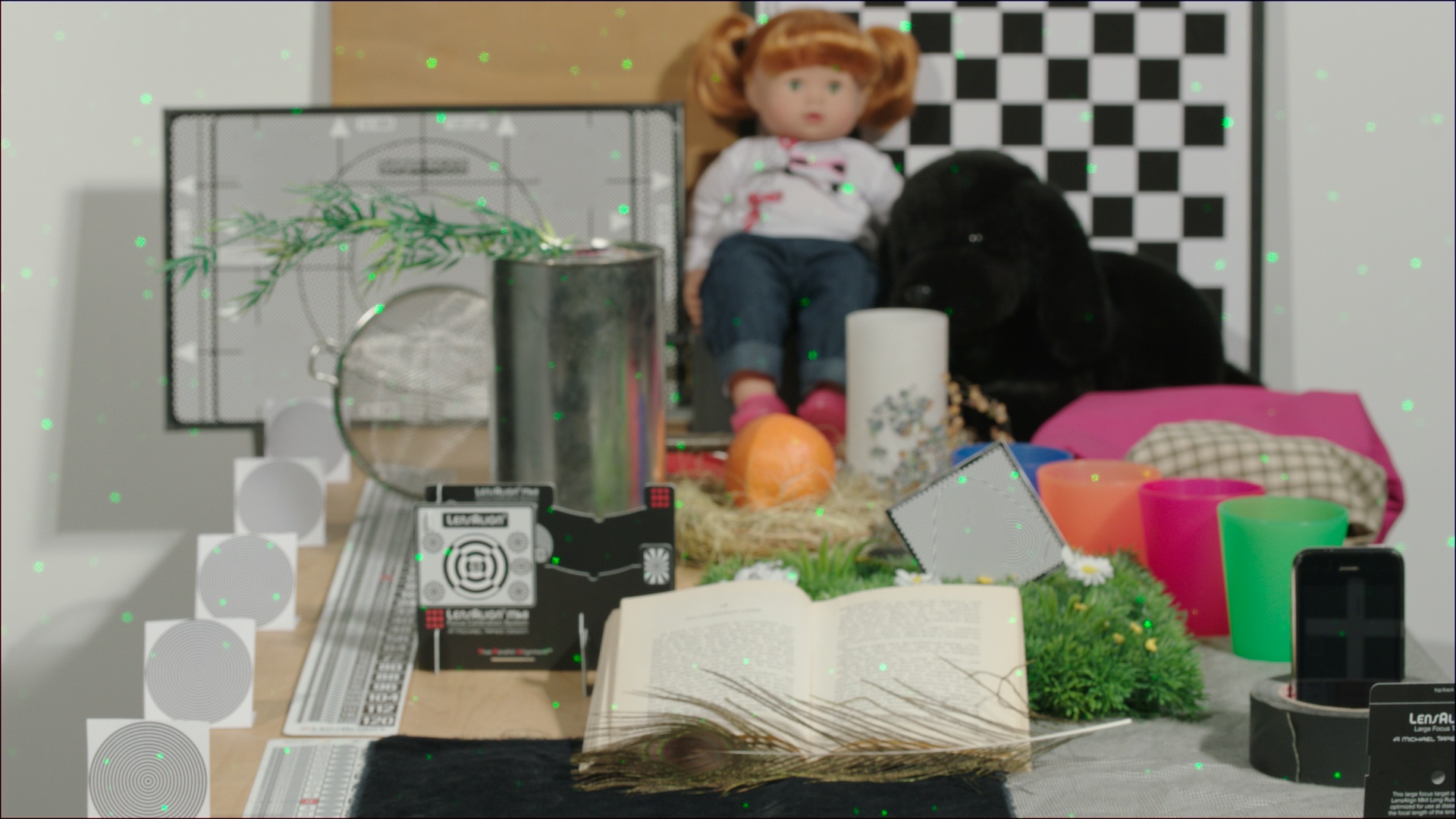}
  }
  \subfloat[\scriptsize Proposed ($L=2$)]{
    \includegraphics[width=0.156\textwidth]{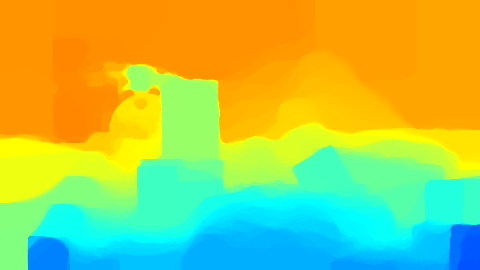}
  }
  \subfloat[\scriptsize Proposed ($L=4$)]{
    \includegraphics[width=0.156\textwidth]{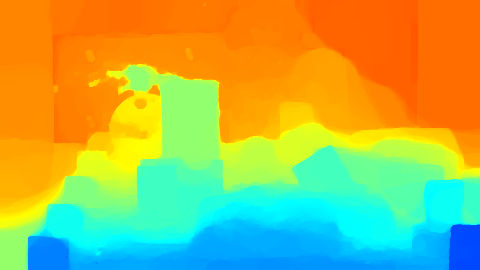}
  }
  \subfloat[\scriptsize Proposed ($L=8$)]{
    \includegraphics[width=0.156\textwidth]{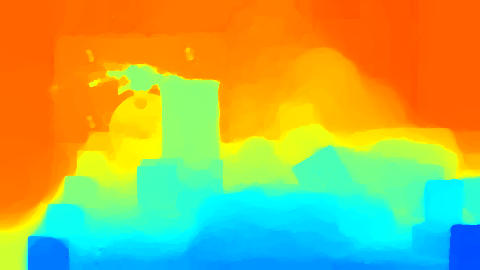}
  }
  \subfloat[\scriptsize Proposed ($L=16$)]{
    \includegraphics[width=0.156\textwidth]{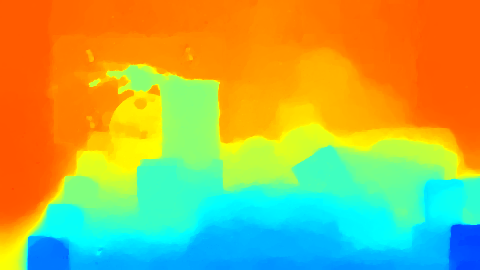}
  }
  \subfloat[\scriptsize Proposed ($L=32$)]{
    \includegraphics[width=0.156\textwidth]{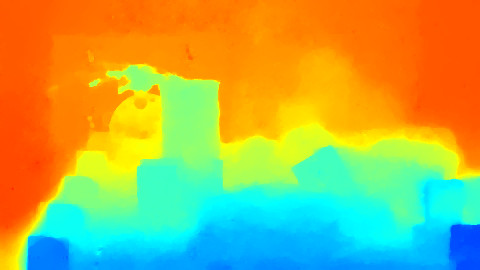}
  }\\[-0.24cm]
  \subfloat[\scriptsize Baseline ($L=374$)]{
    \includegraphics[width=0.156\textwidth]{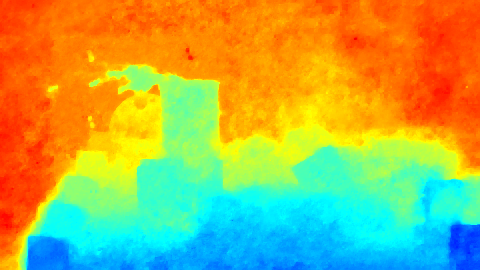}
  }
  \subfloat[\scriptsize Baseline ($L=2$)]{
    \includegraphics[width=0.156\textwidth]{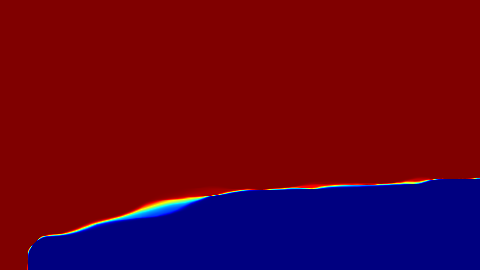}
  }
  \subfloat[\scriptsize Baseline ($L=4$)]{
    \includegraphics[width=0.156\textwidth]{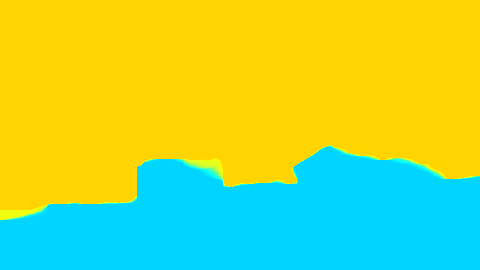}
  }
  \subfloat[\scriptsize Baseline ($L=8$)]{
    \includegraphics[width=0.156\textwidth]{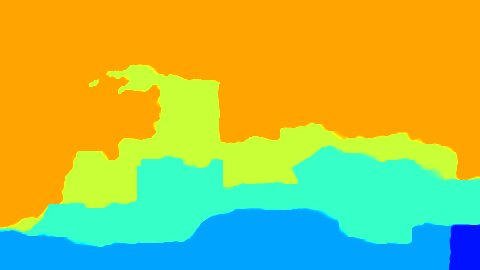}
  }
  \subfloat[\scriptsize Baseline ($L=16$)]{
    \includegraphics[width=0.156\textwidth]{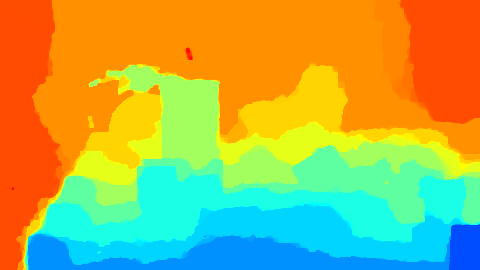}
  }
  \subfloat[\scriptsize Baseline ($L=32$)]{
    \includegraphics[width=0.156\textwidth]{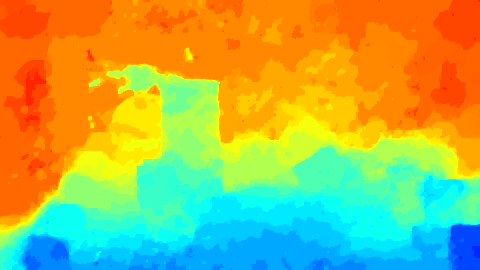}
  }\\[1mm]
  \caption{Depth from focus comparison. We compare our
    method to the baseline approach on the problem of depth from focus. First column: one of the 374 differently
    focused input images and the baseline method for full number of
    labels. Following columns: proposed relaxation (top row)
    vs. baseline (bottom row) for 2, 4, 8, 16 and 32 labels each.}
  \label{fig:dff_compare}
\end{figure*}
}

\newcommand{\figLifting}{
  \begin{figure}[t]
    \centering
    \includegraphics{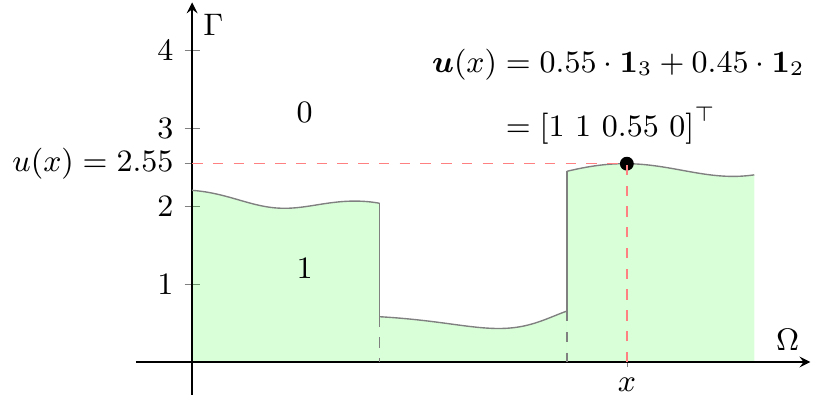}
    \caption{Lifted representation. Instead of optimizing over the function $u : \Omega \rightarrow \Gamma$, we optimize
      over all possible graph functions $\ul : \Omega \rightarrow \bbR^k$ (in this example $k=4$). The central idea behind
      our approach is the finite dimensional representation of $\ul$ at
      every point $x \in \Omega$. \vspace{-0.25cm}}
    \label{fig:graph_intuition}
  \end{figure}
}

\newcommand{\figEpi}{
  \begin{figure}[t!]
    \centering
    \captionsetup[subfloat]{justification=centering,singlelinecheck=false}
    \subfloat[][]{
      \includegraphics{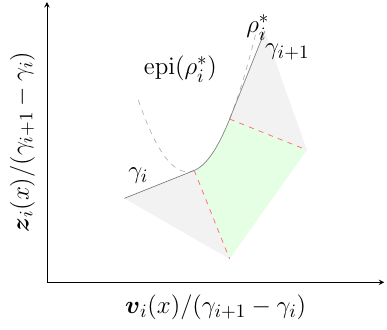}
    }
    \subfloat[][]{
      \includegraphics{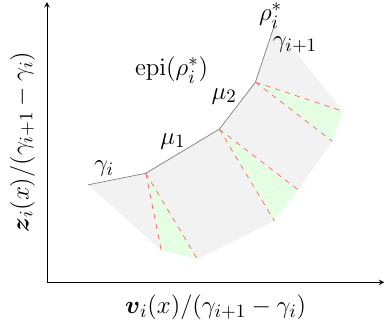}
    }
    \caption{Illustration of the epigraph projection. In the left subfigure the projection onto the epigraph of the conjugate of a convex quadratic $\rho_i$ is shown. In the right subfigure the piecewise linear case is illustrated. In the both cases all points that lie in the gray sets are orthogonally projected onto the respective linear parts whereas the points that lie in the green sets are projected onto the parabolic part (in the quadratic case) respectively the kinks (in the piecewise linear case). In the piecewise linear case the green sets are normal cones. The red dashed lines correspond to the boundary cases. $\gamma_i$, $\gamma_{i+1}$, $\mu_1$, $\mu_2$ are the slopes of the segments of $\rho_i^*$ respectively the (sub-)label positions of $\rho_i$.}
    \label{fig:epigraph_projection}
  \end{figure}
}

\cvprfinalcopy % *** Uncomment this line for the final submission

\def\cvprPaperID{1458} % *** Enter the CVPR Paper ID here

\ifcvprfinal\fi

\begin{document}

%%%%%%%%% TITLE
\title{Sublabel--Accurate Relaxation of Nonconvex Energies}

\newcommand*\samethanks[1][\value{footnote}]{\footnotemark[#1]}
\author[1]{Thomas M\"ollenhoff \thanks{Those authors contributed equally.}}
\author[1]{Emanuel Laude \samethanks}
\author[1]{Michael Moeller}
\author[2]{Jan Lellmann}
\author[1]{Daniel Cremers}
\affil[1]{Technical University of Munich,}
\affil[2]{University of L\"ubeck}

\renewcommand\Authands{ and }

\maketitle
% \thispagestyle{empty}

%%%%%%%%% ABSTRACT
\begin{abstract}
  % In this work we propose a novel convex relaxation of multilabel problems with continuous label sets. Based on the observation, that continuous real-valued labels can be expressed in terms of convex combinations of two consecutive discrete samples of the continuous label set $\Gamma$, we found that this can be incorporated into our model in terms of an $\mathit{SOS2}$-constraint (special ordered set of type 2). By computing the biconjugate of our data term we obtained its largest convex under-approximation, its convex envelope which has the nice property of containing the set of original minimizers as a subset of its minimizers. Moreover, we compute a locally tight convex approximation of a total variation like regularizer, that involves infinitely many constraints. However we show, that the number of required constraints can be reduced to linearly (in the number of label-samples) many. Although globally optimal convex continuous label set approaches have already been presented in the literature, they are formulated in terms of continuous graph functions defined on $\Omega \times \Gamma$ where $\Omega$ is the image domain. However, when discretized, those approaches loose the property of being sublabel accurate. For that reason our method outperforms those approaches, in the sense that it maintains sublabel accuracy, even when it is discretized and implemented on a computer. We demonstrate the effectiveness of our method in a number of different standard computer vision experiments.
  % \jl{der abstract ist noch zu lang}

  We propose a novel spatially continuous framework for convex
  relaxations based on functional lifting. Our method can be
  interpreted as a sublabel--accurate solution to multilabel problems. 
  We show that previously proposed functional
  lifting methods optimize an energy which is linear between two
  labels and hence require (often infinitely) many labels for a
  faithful approximation.  In contrast, the proposed formulation is
  based on a piecewise convex approximation and therefore needs far
  fewer labels -- see Fig.~\ref{fig:teaser}. In comparison to recent
  MRF-based approaches, our method is formulated in a spatially
  continuous setting and shows less grid bias.  
  Moreover, in a local sense, our formulation is the tightest possible
  convex relaxation.  It is easy to implement and allows an efficient
  primal-dual optimization on GPUs. We show the effectiveness
  of our approach on several computer vision problems.

\end{abstract}

%%%%%%%%% BODY TEXT

\section{Introduction}
%Many interesting problems in computer vision such as stereo matching or optical flow estimation inherently lead to nonconvex energies. 
Energy minimization methods have become the central paradigm for
solving practical problems in computer vision. The energy functional
can often be written as the sum of a data fidelity and a
regularization term. One of the most popular regularizers is the total
variation ($\TV$) due to its many favorable properties
\cite{ChambolleTV-2010}. Hence, an important class of optimization
problems is given as
\begin{equation}
  \underset{u : \Omega \rightarrow \Gamma} \min ~ \int_{\Omega}
  \rho(x, u(x)) ~ \mathrm{d}x + \lambda ~ \TV(u),
  \label{eq:general_energy}
\end{equation}
defined for functions $u $ with finite total variation, arbitrary,
possibly nonconvex dataterms $\rho : \Omega \times \Gamma \rightarrow
\bbR$, label spaces $\Gamma$ which are closed intervals in $\bbR$,
$\Omega \subset \bbR^d$, and $\lambda \in \bbR^+$. The multilabel
interpretation of the dataterm is that $\rho(x,u(x))$ represents the
costs of assigning label $u(x)$ to point $x$. For (weakly)
differentiable functions $TV(u)$ equals the integral over the norm of
the derivative, and therefore favors a spatially coherent label
configuration. The difficultly of minimizing the nonconvex energy
\eqref{eq:general_energy} has motivated researchers to develop convex
reformulations.

\begin{figure}
  \captionsetup[subfloat]{labelformat=empty,justification=centering,singlelinecheck=false,margin=0pt}
  \subfloat[][\scriptsize Pock \emph{et al.} \cite{PCBC-SIIMS}, 48 labels, $1.49$ GB, $52s$.]{
    \includegraphics[width=0.23\textwidth]{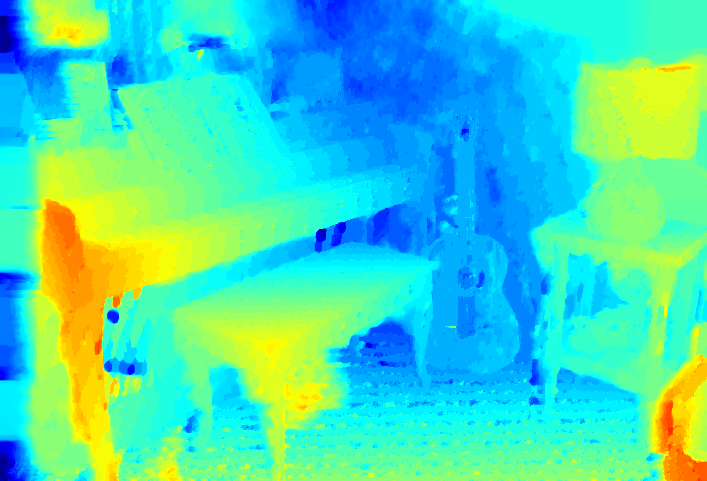}
  }
  \subfloat[][\scriptsize Proposed, 8 labels, $0.49$ GB, $30s$.]{
    \includegraphics[width=0.23\textwidth]{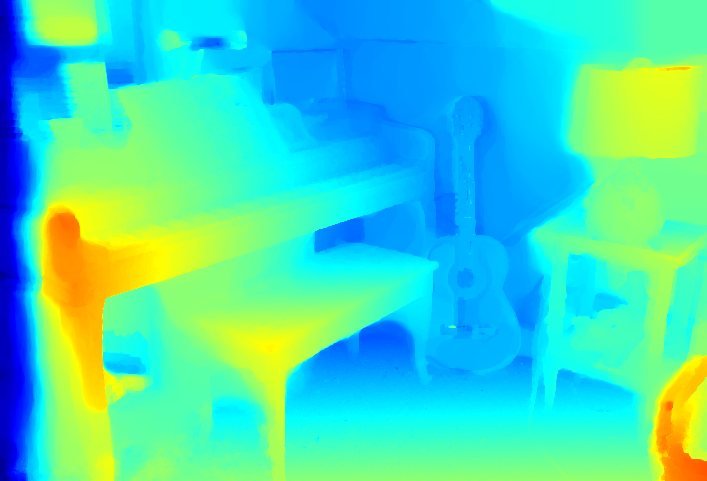}
  }\\[1mm]
  
  \caption{\label{fig:teaser} We propose a convex relaxation for the
    variational model \eqref{eq:general_energy}, which opposed to
    existing functional lifting methods \cite{PCBC-SIIMS,PockECCV}
    allows continuous label spaces \emph{even after} discretization.
    Our method (here applied to stereo matching) avoids label
    space discretization artifacts, while saving on memory and
    runtime.}
\end{figure}

Convex representations of \eqref{eq:general_energy} and more general
related energies have been studied in the context of the
calibration method for the Mumford-Shah functional
\cite{Alberti-et-al-03}. Based on these works, relaxations for
the piecewise constant \cite{Pock-et-al-cvpr09} and piecewise smooth
Mumford-Shah functional \cite{PCBC-ICCV09} have been proposed.
Inspired by Ishikawa's graph-theoretic globally optimal solution to
discrete variants of \eqref{eq:general_energy}, continuous analogues
have been considered by Pock \etal in \cite{PCBC-SIIMS,PockECCV}.
Continuous relaxations for multilabeling problems with finite label
spaces $\Gamma$ have also been studied in~\cite{Lellmann-Schnoerr-siims11}.

Interestingly, the discretization of the aforementioned continuous
relaxations is very similar to the linear programming relaxations
proposed for MAP inference in the Markov Random Field (MRF) community
\cite{Ishikawa,Schlesinger76,Werner-tpami2007,Zach-et-al-cvpr12}.
Both approaches ultimately discretize the range $\Gamma$ into a finite
set of labels. A closer analysis of these relaxations reveals,
however, that they are not well-suited to represent the 
continuous valued range that we face in most computer vision problems
such as stereo matching or optical flow.  More specifically, the above
relaxations are not designed to assign meaningful cost values to
non-integral configurations. As a result, a large number of labels is
required to achieve a faithful approximation.  Solving real-world
vision problems therefore entails large optimization problems with
high memory and runtime requirement.  To address this problem, Zach
and Kohli \cite{Zach-Kohli-eccv12}, Zach \cite{Zach-aistats13} and Fix
and Agarwal \cite{Fix-eccv14} introduced MRF-based approaches which
retain continuous label spaces after discretization. For
manifold-valued labels, this issue was addressed by Lellmann~\etal
\cite{lellmann-et-al-iccv2013}, however with the sole focus on the
regularizer.

\subsection{Contributions}
We propose the first sublabel--accurate convex relaxation
of nonconvex problems in a spatially continuous setting.  It exhibits several favorable
properties:\\[-6mm]

\begin{itemize}
\setlength{\itemsep}{1mm}
\item In contrast to existing spatially continuous lifting
  approaches \cite{PCBC-SIIMS,PockECCV}, the proposed method provides
  substantially better solutions with far fewer labels -- see
  Fig.~\ref{fig:teaser}.  This provides savings in runtime and
  memory.

  \item In Sec.~\ref{sec:lifting} we show that the functional
    lifting methods \cite{PCBC-SIIMS,PockECCV} are a special case of
    the proposed framework. 
    
 \item In Sec.~\ref{sec:lifting} we show that, in a local sense, our formulation is the tightest convex relaxation which takes dataterm and regularizer into
  account separately. It is unknown whether this
  ``local convex envelope'' property also holds for the discrete
  approach~\cite{Zach-Kohli-eccv12}.

\item Our formulation is compact and requires only half
  the amount of variables for the dataterm than the formulation
   in \cite{Zach-Kohli-eccv12}. We prove that the sublabel--accurate total variation can be represented in a very simple way,
  introducing no overhead compared to \cite{PCBC-SIIMS,PockECCV}.  In
  contrast, the regularizer in \cite{Zach-Kohli-eccv12} is much more
  involved.

\item Since our method is derived in a spatially continuous setting,
  the proposed approach easily allows different gradient
  discretizations. In contrast to
  \cite{Zach-aistats13,Zach-Kohli-eccv12} the regularizer is
  isotropic leading to noticeably less grid bias.
\end{itemize}

% Wir verallgemeinern nun Zach ins kontinuierliche

% \jl{- Isotrope Regularisierung}
% \jl{- Wir zeigen dass man den Regularisierer der Standard-Relaxierung nehmen kann, keine komplizierte Perspective/Epigraph projection fuer Regularizer notwendig wie bei Zach ECCV12}
% \jl{- Unsere Formulierung reduziert sich auf das ungeliftete Problem fuer den 2 Label Fall. Wir schlagen damit eine Bruecke zwischen ``ungelifteter'' konvexer Optimierung und den Functional Lifting Verfahren.}
% \jl{- Unsere Formulierung reduziert sich fuer stueckweise Lineare Stuecke auf das Standard Lifting (Beweis fuer Graphbasis fehlt noch). Somit ist es eine direkte Erweiterung. }
% \jl{- Wir leiten den beweisbar tightesten Datenterm her (fuer die Formulierung, Zusammenhang mit Zach unklar). Unsere Formulierung hat nur eine Variable, Zach hat zwei. }
% \jl{- An efficient GPU implementation allows to solve large-scale problems like stereo matching on images consisting of several megapixels within a convex framework}
% \jl{- Konvexe Optimierung von SOS2 Constraint}

\section{Notation and Mathematical Preliminaries}
We make heavy use of the convex conjugate, which is given as $f^*(y) =
\sup_{x \in \bbR^n} ~ \iprod{y}{x} - f(x)$ for functions $f : \bbR^n
\rightarrow \bbR \cup \{ \infty \}$. The biconjugate $f^{**}$ denotes
its \emph{convex envelope}, i.e.~the largest lower-semicontinuous
convex under-approximation of $f$.  For a set $C$ we denote by
$\delta_C$ the function which maps any element from $C$ to $0$ and is
$\infty$ otherwise.  For a comprehensive introduction to convex
analysis, we refer the reader to \cite{Rockafellar:ConvexAnalysis}.
Vector valued functions $\ul : \Omega \to \bbR^k$ are written in bold
symbols. If it is clear from the context, we will drop the $x \in
\Omega$ inside the functions, e.g., we write $\rho(u)$ for $\rho(x,
u(x))$.

\section{Functional Lifting}
\figLifting

\begin{figure*}[t!]
  \subfloat[][]{
    \newlength\fheight 
    \newlength\fwidth 
    \setlength\fheight{1.25cm} 
    \setlength\fwidth{8.3cm}
    {\scalefont{0.8}
    \input{stereo_standard_energy.tex}
    }
    \label{fig:standard_energy}
  }
  \hspace{-0.7cm}
  \subfloat[][]{
    \setlength\fheight{1.25cm} 
    \setlength\fwidth{8.3cm}
    {\scalefont{0.8}
    \input{stereo_precise_energy.tex}
    }
    \label{fig:precise_energy}
  }\\[1mm]
   \vspace{-0.3cm}
  \caption{
    We show the nonconvex energy $\rho(u)$ at a fixed point $x \in \Omega$ (red
    dashed line in both plots) from the stereo matching experiment in
    Fig.~\ref{fig:stereo_compare} over the full range
    of $270$ disparities. The black dots indicate the positions of the labels and the black curves show the approximations used by the respective methods. Fig.~\ref{fig:standard_energy}: The baseline
    lifting method \cite{PCBC-SIIMS} uses a piecewise linear approximation with labels as nodes.
    Fig.~\ref{fig:precise_energy}: The proposed method uses an optimal piecewise convex approximation. As we can see, the piecewise convex approximation is closer to the original nonconvex energy and therefore more accurate.  }
  \label{fig:stereo_energy}
\end{figure*}

\label{sec:lifting}
To derive a convex representation of \eqref{eq:general_energy}, we rely on the framework of functional
lifting. The idea is to reformulate the
optimization problem in a higher dimensional space, in which the
convex envelope approximates the nonconvex energy better
than the one of the original low dimensional energy. 
We start by sampling the range $\Gamma$ at $L = k+1$ labels
$\gamma_1 <\hdots <\gamma_L \in \Gamma$. This partitions
the range into $k$ intervals $\Gamma_i = [\gamma_i, \gamma_{i+1}]$ so
that $\Gamma = \Gamma_1 \cup \hdots \cup \Gamma_k$. Clearly, any value
in the range of $u : \Omega \to \Gamma$ can be written as 
\begin{equation}
u(x) = \gamma_i^\alpha := \gamma_i + \alpha (\gamma_{i+1} - \gamma_i),
\end{equation}
 for $\alpha \in [0,1]$ and some label index $1 \leq i \leq k$. We represent such a value in the range $\Gamma$ by a $k$-dimensional vector
\begin{equation}
  \boldsymbol u(x)=\Gr_i^{\alpha} := \alpha \Gr_i + (1 - \alpha) \Gr_{i-1},
\end{equation}
where $\Gr_i$ denotes a vector starting with $i$ ones followed by $k -
i$ zeros. We call $\ul : \Omega \to \bbR^k$ the \emph{lifted} representation of $u$, representing the graph of $u$. This notation is depicted in Fig.~\ref{fig:graph_intuition} for $k = 4$. 
% The central idea is that at every point $x\in \Omega$ we can represent $\ul(x) \in \bbR^k$ using a finite dimensional vector, so we can retain a continuous label space even after discretization.
Back-projecting the lifted
$\boldsymbol u(x)$ to the range of $u$ using the layer cake formula yields a one-to-one correspondence between $u(x) = \gamma_i^\alpha$ and $\boldsymbol u(x) = \Gr_i^{\alpha}$ via
\begin{equation}
  u(x)= \gamma_1 + \sum_{i=1}^k {\boldsymbol u}_i(x) (\gamma_{i+1} - \gamma_i).
\end{equation}
%such that there is a one-to-one correspondence between $u(x) = \gamma_i^\alpha$ and $\boldsymbol u(x) = \Gr_i^{\alpha}$.
%We represent the graph 
%function at every point using a \emph{finite dimensional} vector
%$\ul(x) \in \bbR^k$, allowing a continuous range \emph{even after} discretization.
We now formulate problem \eqref{eq:general_energy} in terms of such
graph functions, a technique that is common in the theory of Cartesian
currents \cite{GMS-CC}.

\subsection{Convexification of the Dataterm}
For now, we consider a fixed $x \in \Omega$. Then the dataterm from \eqref{eq:general_energy}
is a possibly nonconvex real-valued function (cf. Fig.~\ref{fig:stereo_energy}) that we seek to minimize over a compact interval $\Gamma$:
\begin{equation}
  \underset{u \in \Gamma} \min ~ \rho(u).
  \label{eq:dataterm_min_simple}
\end{equation}
%Using the lifted representation $\ul : \Omega \rightarrow \bbR^k$ we
%can rewrite the dataterm from \eqref{eq:general_energy} 
%at every point $x \in \Omega$ as 
Due to the one-to-one correspondence between $\gamma_i^\alpha$ and $\Gr_i^{\alpha}$ it is clear that solving problem \eqref{eq:dataterm_min_simple} is equivalent to finding a
minimizer of the lifted energy: %$\dat : \Omega \times \bbR^k \to \bbR \cup \{ \infty \}$ given as:
\begin{equation}
  \dat(\ul) = \underset{1 \leq i \leq k} \min ~ \dat_i(\ul),
  \label{eq:dataterm_lifted}
\end{equation}
%where the individual $\dat_i$ are given as
\begin{equation}
  \dat_i(\ul) = 
  \begin{cases}
    \rho(\gamma_i^\alpha), \qquad &\text{if } ~ \ul = \Gr_i^{\alpha}, ~ \alpha \in [0, 1],\\
    \infty, & \text{else.}
  \end{cases}
  \label{eq:dataterm_lifted_part}
\end{equation}
%The lifted function $\dat$ is still nonconvex, since its domain which is the set of all graph functions is nonconvex and because the function $\rho$ might also be nonconvex. 
%Due to the  solving \eqref{eq:dataterm_lifted} is equivalent to \eqref{eq:dataterm_min_simple}.
Note that the constraint in \eqref{eq:dataterm_lifted_part} is essentially the nonconvex special ordered set of type 2 (SOS2) constraint \cite{Beale-Tomlin-1970}. More precisely, we demand that the derivative $\partial_{\gamma} \ul$ is zero, except for two neighboring elements, which add up to one. 
In the following proposition, we derive the tightest convex relaxation of $\dat$.
\begin{prop}
  The convex envelope of \eqref{eq:dataterm_lifted} is given as:
  \begin{equation} \label{eq:dataterm_sublabel_biconj}
    \dat^{**}(\ul) = \underset{\vl \in \bbR^k} \sup ~ \iprod{\ul}{\vl} - \underset{1 \leq i \leq k} \max ~ \dat_i^*(\vl), 
  \end{equation}
  where the conjugate of the individual $\dat_i$ is%can be expressed by the conjugate of the 1-dimensional energy via
  \begin{equation}
      \dat_i^*(\vl) = c_i(\vl) + \rho_i^*\left(\frac{\vl_i}{\gamma_{i+1} - \gamma_i}\right),
      \label{eq:dataterm_optimize}
  \end{equation}
  with $c_i(\vl) = \iprod{\Gr_{i-1}}{\vl} - \frac{\gamma_i}{\gamma_{i+1} - \gamma_i}$ and $\rho_i=\rho + \delta_{\Gamma_i}$. 
\end{prop}
\begin{proof}
  See appendix.
\end{proof}
The above proposition reveals that the convex relaxation implicitly
convexifies the dataterm $\rho$ on each interval $\Gamma_i$. The equality $\rho_i^*=\rho_i^{***}$ implies that starting with $\rho_i$ yields exactly the same convex relaxation as starting with $\rho_i^{**}$. 
%By considering an additional constraint $\delta_{\{\Gr_i \Gr_{i+1}\}}$ for each $\dat_i$ 
%This allows to draw the following interesting conclusion.
%We can conclude:
\begin{cor}\label{prop:piecewiseLinear}
If $\rho$ is linear on each $\Gamma_i$,
% i.e., if it holds for all $1 \leq i \leq k$ and all $\alpha \in [0,1]$:
%\begin{equation}
%\rho(\alpha \gamma_{i+1} + (1 - \alpha) \gamma_i) = \alpha \rho(\gamma_{i+1}) + (1 - \alpha) \rho(\gamma_i),
%\end{equation}
then the convex envelopes of $\dat(\ul)$ and $\dats(\ul)$ coincide, where the latter is:
\begin{equation}
\dats(\ul) =
 \begin{cases}
    \rho(\gamma_i^\alpha), \qquad &\text{if } ~\exists i: \; \ul = \Gr_i^{\alpha}, ~ \alpha \in \{0, 1\},\\
    \infty, & \text{else.}
  \end{cases}
\end{equation}
\end{cor}
\begin{proof}
Consider an additional constraint $\delta_{\{\gamma_i, \gamma_{i+1}\}}$ for each $\rho_i$, which corresponds to selecting $\alpha \in \{ 0, 1 \}$ in \eqref{eq:dataterm_lifted_part}. The fact that our relaxation is independent of whether we choose $\rho_i$ or $\rho_i^{**}$, along with the fact that the convex hull of two points is a line, yields the assertion.
\end{proof}
\vspace{-0.2cm}
%One can not only state the appearance of the biconjugate of the special case discussed above explicitly
For the piecewise linear case, it is possible to find an explicit form of the biconjugate.
%Therefore, our dataterm can be interpreted as a natural extension of the dataterm in~\cite{Pock-et-al-cvpr09}.
\begin{prop} \label{prop:standardRelation}
Let us denote by $\textbf{r} \in \bbR^k$ the vector with
\begin{equation}
\textbf{r}_i = \rho(\gamma_{i+1}) - \rho(\gamma_{i}), \ \ 1 \leq i \leq k. %i \in \{1, ..., k\}
\end{equation}
Under the assumptions of Prop.~\ref{prop:piecewiseLinear}, one obtains:
\begin{equation}
  \dats^{**}(\ul) = 
  \begin{cases}
     \rho(\gamma_1) + \langle \ul, \textbf{r} \rangle, \qquad &\text{if } ~ \ul_{i} \geq \ul_{i+1}, \ul_i \in [0,1],\\
    \infty, & \text{else.}
  \end{cases}
  \label{eq:dataterm_standard_lifted_relaxed}
\end{equation}
\end{prop}
\begin{proof}
  See appendix.
\end{proof}
Up to an offset (which is irrelevant for the optimization), one can see that \eqref{eq:dataterm_standard_lifted_relaxed} coincides with the dataterm of~\cite{Pock-et-al-cvpr09}, the discretizations of~\cite{PCBC-SIIMS,PockECCV}, and -- after a change of variable -- with \cite{Lellmann-Schnoerr-siims11}. This not only proves that the latter is optimizing a convex envelope, but also shows that our method naturally generalizes the work from piecewise linear to arbitrary piecewise convex energies. Fig.~\ref{fig:standard_energy} and Fig.~\ref{fig:precise_energy} illustrate the difference of $\dats^{**}$ and $\dat^{**}$ on the example of a nonconvex stereo matching cost.

Because our method allows arbitrary convex functions on each $\Gamma_i$, we can prove that, for the two label case, our approach optimizes the convex envelope of the dataterm. 
%We show this in the following proposition:
\begin{prop}
  In the case of binary labeling, i.e., $L = 2$, the convex envelope of \eqref{eq:dataterm_lifted}
  reduces to 
  \begin{equation}
    \dat^{**}(\ul) = \rho^{**}\left( \gamma_1 + \ul (\gamma_2 - \gamma_1) \right), \text{ with } \ul \in [0, 1].
  \end{equation}
  \vspace{-0.5cm}
  \label{prop:unlifted}
\end{prop}
\begin{proof}
  See appendix.
\end{proof}
\subsection{A Lifted Representation of the Total Variation}
We now want to find a lifted convex formulation that emulates the total
variation regularization in \eqref{eq:general_energy}. We follow \cite{Chambolle-et-al-siims12} and define an appropriate integrand
of the functional
\begin{equation}
  TV(\ul) = \int_{\Omega} \reg(x, D \ul),
\end{equation}
where the distributional derivative $D \ul$ is a finite $\bbR^{k \times d}$-valued Radon measure \cite{Ambrosio-et-al:BV}.  We define
%The sublabel accurate integrant $\reg : \bbR^{k \times d} \rightarrow \bbR \cup \{ \infty \}$ is given as:
%\jl{Write more here, use stuff about 1-homogenity from Jan's mail.} We formulate this as follows:
\begin{equation}
\reg(\gl) = \underset{1 \leq i \leq j \leq k} \min ~ \reg_{i,j}(\gl).
  \label{eq:regularizer_lifted}
\end{equation}
The individual $\reg_{i,j} : \bbR^{k \times d} \to \bbR \cup \{\infty \}$ are given by:
\begin{equation}
  \reg_{i,j}(\gl) = 
  \begin{cases}
    \left|\gamma_i^{\alpha} - \gamma_j^{\beta}\right| \cdot \normc{\nu}_2, \qquad & \text{ if } \gl = (\Gr_i^{\alpha} - \Gr_j^{\beta}) ~ \nu^{\mathsf{T}},\\
    \infty, & \text{ else,}
  \end{cases}
  \label{eq:regularizer_single}
\end{equation}
for some $\alpha, \beta \in [0,1]$ and $\nu \in \bbR^d$. The intuition is that $\Phi_{i,j}$ penalizes a jump from $\gamma_i^{\alpha}$ to $\gamma_j^{\beta}$ in the direction of $\nu$. Since $\reg$ is nonconvex we compute the convex envelope. 
\vspace{-2mm}
\begin{prop}
  The convex envelope of \eqref{eq:regularizer_lifted} is 
  \begin{equation}
    \reg^{**}(\gl) = \underset{\ql \in \mathcal{K}} \sup ~
    \iprod{\ql}{\gl},
  \end{equation}
  \vspace{-2mm}
  where $\mathcal K \subset \bbR^{k \times d}$ is given as:
  \begin{equation}
    \begin{aligned}
      \mathcal{K} = \left\{ \ql \in \vphantom{\left| \ql^{\mathsf{T}} ( \Gr_i^{\alpha} - \Gr_j^{\beta} )\right|_2}\right.&\bbR^{k \times d} ~\Big\vert~ \\ 
      & \left| \ql^{\mathsf{T}} ( \Gr_i^{\alpha} - \Gr_j^{\beta} )\right|_2 \leq \left| \gamma_i^{\alpha} - \gamma_j^{\beta}\right|, \\
      &  \left.\vphantom{\left| \ql^{\mathsf{T}} ( \Gr_i^{\alpha} - \Gr_j^{\beta} )\right|_2}\forall ~ 1 \leq i \leq j \leq k, ~ \forall \alpha, \beta \in [0, 1] \right\}.
    \end{aligned}
    \label{eq:constraints_infinite}
  \end{equation}
  \label{prop:regularizer}
  \vspace{-0.5cm}
\end{prop}
\begin{proof}
  See appendix.
\end{proof}
\vspace{-2mm}
\figEpi
The set $\mathcal{K}$ from Eq.~\eqref{eq:constraints_infinite} involves infinitely 
many constraints which makes numerical optimization difficult. 
%A possible strategy would be to implement the constraints for some sampled $\alpha, \beta \in [0,1]$.
As the 
following proposition reveals, the infinite number of constraints can be reduced to only linearly many, allowing to enforce the constraint $\ql \in \mathcal{K}$ exactly. 
\vspace{-2mm}
\begin{prop}
  In case the labels are ordered, i.e., $\gamma_1 < \gamma_2 < \hdots
  < \gamma_L$, then the constraint set
  $\mathcal{K}$ from Eq.~\eqref{eq:constraints_infinite} is equal to 
  \vspace{-2mm}
  \begin{equation}
    \begin{aligned}
      \mathcal{K} = \{ \ql \in &\bbR^{k \times d} ~\mid~ \normc{\ql_i}_2 \leq \gamma_{i+1} - \gamma_i,  ~ \forall i \}.
    \end{aligned}
  \end{equation}
  \vspace{-0.5cm}
  \label{prop:reduction}
\end{prop}
\begin{proof}
  See appendix.
\end{proof}
This shows that the proposed regularizer coincides with the total variation from \cite{Chambolle-et-al-siims12}, where it has been derived based on \eqref{eq:regularizer_single} for $\alpha$ and $\beta$ restricted to $\{0,1\}$. 
%$$
%TV(\ul) = \underset{\ql \in C_c^{\infty}(\Omega, \bbR^{d \times k}), \ql(x) \in \mathcal{K}} \sup ~ \int_{\Omega} \iprod{\ul}{\Div \ql} ~ \mathrm{d}x.
%$$
Prop.~\ref{prop:reduction} together with Prop.~\ref{prop:unlifted} show that for $k=1$ our formulation
amounts to unlifted $\TV$ optimization with a convexified dataterm.

\section{Numerical Optimization}
\figROF
\figTruncatedROF

Discretizing $\Omega\subset \bbR^d$ as a $d$-dimensional Cartesian grid, the relaxed energy minimization problem becomes
\begin{equation} \label{eq:discrete_saddle_point}
\min_{\ul : \Omega \to \mathbb{R}^k} \sum_{x \in \Omega} \dat^{**}(x, \ul(x)) + \reg^{**}(x, \nabla \ul(x)),
\end{equation}
where $\nabla$ denotes a forward-difference operator with $\nabla \ul : \Omega \to \mathbb{R}^{k \times d}$. We rewrite the dataterm given in equation~\eqref{eq:dataterm_sublabel_biconj} by replacing the pointwise maximum over the conjugates $\dat_i^*$ with a maximum over a real number $q \in \mathbb{R}$ and obtain the following saddle point formulation of problem~\eqref{eq:discrete_saddle_point}:
\begin{equation}\label{eq:discrete_saddle_point_final}
\min_{\ul : \Omega \to \mathbb{R}^k} \max_{\substack{(\vl, q) \in \mathcal{C}\\[0.5mm] \boldsymbol p : \Omega \to \mathcal{K}}} \langle \ul, \vl \rangle - \sum_{x \in \Omega} q(x) + \langle \boldsymbol p, \nabla \ul\rangle,
\end{equation}
%where the set $\mathcal{C}$ is given as
\begin{equation}
    \begin{aligned}
	\mathcal{C} = \{ &(\vl, q) : \Omega \to \bbR^{k}\times \bbR \mid q(x) \geq \dat_i^*(\vl(x)), \ \forall x, \forall i\}.
    \end{aligned}
    \label{eq:constraint_C}
\end{equation}
We numerically compute a minimizer of problem~\eqref{eq:discrete_saddle_point_final} using a first-order primal-dual method~\cite{Esser-Zhang-Chan-10,PCBC-ICCV09} with diagonal preconditioning \cite{Pock-Chambolle-iccv11} and adaptive steps \cite{Goldstein-Esser-13}. It basically alternates between a gradient descent step in the primal variable and a gradient ascent step in the dual variable. Subsequently the dual variables are orthogonally projected onto the sets $\mathcal{C}$ respectively $\mathcal{K}$. The projection onto the set $\mathcal{K}$ is a simple $\ell_2$-ball projection.
% the direct projection onto the set $\mathcal{C}$ is difficult to compute. For that reason we transform the $k$-dimensional epigraph constraints
To simplify the projection onto $\mathcal{C}$, we transform the $k$-dimensional epigraph constraints in \eqref{eq:constraint_C}
%\begin{equation} \label{eq:high-dim-epigraph-constrs}
%q(x) \geq \dat_i^*(\vl(x)),
%\end{equation}
into $1$-dimensional scaled epigraph constraints by introducing an additional variable $z:\Omega \to \mathbb{R}^k$ with:
\begin{equation} \label{eq:equality_constr_data}
  \begin{aligned}
    \boldsymbol z_i(x)= \left[ q(x) - c_i\left(\vl(x)\right) \right] \left( \gamma_{i+1} - \gamma_i \right).
  \end{aligned}
\end{equation}
Using equation~\eqref{eq:dataterm_optimize} we can now rewrite the constraints in~\eqref{eq:constraint_C} as
\begin{equation}
  \frac{ \boldsymbol z_i(x)}{\gamma_{i+1} - \gamma_i} \geq \rho_i^*\left( \frac{\vl_i(x)}{\gamma_{i+1} - \gamma_i}  \right).
\end{equation}
We implement the newly introduced equality constraints~\eqref{eq:equality_constr_data} introducing a Lagrange multiplier $\boldsymbol s : \Omega \to \mathbb{R}^k$. It remains to discuss the orthogonal projections onto the epigraphs of the conjugates $\rho_i^*$. Currently we support quadratic and piecewise linear convex pieces $\rho_i$. For the piecewise linear case, the conjugate $\rho_i^*$ is a piecewise linear function with domain $\bbR$. The slopes correspond to the $x$-positions of the sublabels and the intercepts correspond to the function values at the sublabel positions.

The conjugates as well as the epigraph projections of both, a quadratic and a piecewise linear piece are depicted in Fig.~\ref{fig:epigraph_projection}. For the quadratic case, the projection onto the epigraph of a parabola is computed using \cite[Appendix B.2]{strekalovskiy-et-al-siims14}.
\figZachKohli
%This results in the problem of estimating the roots of a third order polynomial. Following~ one can exploit the special structure of the problem and simplify the computation of the projection.
\vspace{-0.2cm}
\section{Experiments}
We implemented the primal-dual algorithm in CUDA to run on GPUs.
For $d=2$, our implementation of the functional lifting framework
\cite{PCBC-SIIMS}, which will serve as a baseline method, requires $4
N (L - 1)$ optimization variables, while the proposed method requires
$6 N (L-1) + N$ variables, where $N$ is the number of points used to
discretize the domain $\Omega \subset \bbR^{d}$.  
%Practically, for the same number of labels, the proposed method requires roughly $1.5$ times as much memory. 
As we will show, our method requires much fewer labels to yield comparable results,
thus, leading to an improvement in accuracy, memory usage, and
speed.
\figAniso
\figStereo

\subsection{Rudin-Osher-Fatemi Model}
As a proof of concept, we first evaluate the novel relaxation on 
the well-known Rudin-Osher-Fatemi (ROF) model \cite{Rudin-Osher-Fatemi-92}. 
It corresponds to \eqref{eq:general_energy} with the following dataterm:
\begin{equation}
  \rho(x, u(x)) = \left (u(x) - f(x) \right)^2,
  \label{eq:cont_rof}
\end{equation}
where $f : \Omega \rightarrow \bbR$ denotes the input data.
While there is no practical use in applying convex relaxation methods to
an already convex problem such as the ROF model, the purpose of this is
two-fold. Firstly, it allows us to measure the overhead introduced by our method by
comparing it to standard convex optimization methods which do not
rely on functional lifting. Secondly, we can experimentally verify
that the relaxation is tight for a convex dataterm.%, and that even for
%higher label numbers our method converges to the globally optimal solution of the problem. 

In Fig.~\ref{fig:rof_compare} we solve \eqref{eq:cont_rof} directly
using the primal-dual algorithm \cite{Goldstein-Esser-13}, using
the baseline functional lifting method \cite{PCBC-SIIMS} and using 
our proposed algorithm. First, the globally optimal energy was
computed using the direct method with a very high number of
iterations. Then we measure how long each method took to reach 
this global optimum to a fixed tolerance. 

The baseline method fails to reach the global optimum even 
for $256$ labels. While the lifting framework introduces a certain
overhead, the proposed method finds the same globally optimal 
energy as the direct unlifted optimization approach and generalizes
to nonconvex energies.

\subsection{Robust Truncated Quadratic Dataterm}

The quadratic dataterm in \eqref{eq:cont_rof} is often not well suited for real-world data 
as it comes from a pure Gaussian noise assumption and does not model
outliers. We now consider a robust truncated quadratic dataterm:
\begin{equation}
  \rho(x, u(x)) = \frac{\alpha}{2} \min \left\{ (u(x) - f(x))^2, \nu \right\}.
  \label{eq:cont_rof_robust}
\end{equation}
To implement \eqref{eq:cont_rof_robust}, we use a piecewise polynomial 
approximation of the dataterm. 
% If two labels $\gamma_i, \gamma_{i+1}$ lie either completely on the quadratic 
% or the truncated part, we can approximate the dataterm exactly. For labels that lie in different
% parts (e.g. $\gamma_i < f - \sqrt{\nu}$ and $\gamma_{i+1} > f - \sqrt{\nu}$) we chose 
% a linear approximation. We sample the labels equidistantly with $\gamma_{i+1} - \gamma_i < 2 \sqrt{\nu}$
% so the case where the quadratic part is completely skipped cannot occur.
In Fig.~\ref{fig:rof_robust} we degraded the input image with additive Gaussian and 
salt and pepper noise. The parameters in
\eqref{eq:cont_rof_robust} were
chosen as $\alpha = 25$, $\nu = 0.025$ and $\lambda = 1$. The
proposed method requires significantly less labels to find lower
energies than the baseline method.

\subsection{Comparison to the Method of Zach and Kohli}
We remark that Prop.~\ref{prop:regularizer} and
Prop.~\ref{prop:reduction} hold for arbitrary convex
one-homogeneous functionals $\phi(\nu)$ instead of $\normc{\nu}_2$ in
equation \eqref{eq:regularizer_single}. In particular, they hold for the anisotropic total
variation $\phi(\nu) = \normc{\nu}_1$.  This generalization allows us to directly
compare our convex relaxation to the MRF approach of Zach and
Kohli~\cite{Zach-Kohli-eccv12}.
In Fig.~\ref{fig:zach_compare} we show the results of optimizing the
two models entitled ``DC-Linear'' and ``DC-MRF'' proposed in~\cite{Zach-Kohli-eccv12},
and of our proposed method with anisotropic regularization on the
robust truncated denoising energy \eqref{eq:cont_rof_robust}. We
picked the parameters as $\alpha = 0.2$, $\nu = 500$, and $\lambda = 1$. The label space is
also chosen as $\Gamma = [0, 256]$ as described in~\cite{Zach-Kohli-eccv12}.
\figDFF
Overall, all the energies are better than the ones reported
in~\cite{Zach-Kohli-eccv12}. It can be seen from
Fig.~\ref{fig:zach_compare} that the proposed relaxation is competitive to
the one proposed in~\cite{Zach-Kohli-eccv12}. In addition, The proposed
relaxation uses a more compact representation and extends to isotropic
regularizers.
% with linearly many constraints.
To illustrate the advantages of isotropic regularizations, Fig.~\ref{fig:aniso} and Fig.~\ref{fig:iso} 
show a comparison of our proposed method for isotropic and anisotropic regularization in the next section. 
%we compare the proposed
%method with anisotropic and isotropic regularization on the problem of
%stereo matching. One can see that the isotropic formulation leads to
%noticeably less grid bias. 

\vspace{-0.1cm}
\subsection{Stereo Matching}
\label{sec:stereo}
Given a pair of rectified images, the task of
finding a correspondence between the two images can be formulated as
an optimization problem over a scalar field $u : \Omega \rightarrow
\Gamma$ where each point $u(x) \in \Gamma$ denotes the displacement
along the epipolar line associated with each $x \in \Omega$. 
The overall cost functional fits Eq.~\eqref{eq:general_energy}.
In our experiments, we computed $\rho(x, u(x))$ for $270$ disparities on the % high-resolution images 
Middlebury stereo benchmark \cite{Hirschmueller-gcpr14}
in a $4 \times 4$ patch using a truncated sum of absolute gradient
differences. 
%The input to the experiments shown in
%Fig.~\ref{fig:stereo_compare} is a $741 \times 500 \times 270$
%disparity cost volume. 
For the dataterm of the proposed relaxation, we convexify the matching
cost $\rho$ in each range $\Gamma_i$ by numerically computing the convex
envelope using the gift wrapping algorithm. %An example of such a
%convexified stereo energy is shown in
%Fig.~\ref{fig:stereo_energy}. 
% This yields a piecewise linear convex energy
% between each pair of labels $\gamma_i, \gamma_{i+1}$, which can be 
% handled within our framework as described in the previous section.

The first row in Fig.~\ref{fig:stereo_compare} shows the result of the proposed
relaxation using the convexified energy between two labels. The 
second row shows the baseline approach using the same amount of labels.
Even for $L=2$, the proposed method produces a reasonable depth map
while the baseline approach basically corresponds to a two region segmentation.

\vspace{-2mm}
\subsection{Phase Unwrapping}
Many sensors such as time-of-flight cameras or interferometric synthetic aperture radar (SAR) yield
cyclic data lying on the circle $\Sone$. 
Here we consider the task of total variation regularized
unwrapping. 
As is shown on the left in Fig.~\ref{fig:sar_experiments}, the dataterm is a nonconvex function where
each minimum corresponds to a phase shift by $2 \pi$:
%\vspace{-2mm}
\begin{equation}
  \rho \left( x, u(x) \right) = d_{\Sone}\left( u(x), f(x) \right)^2.
\end{equation}
For the experiments, we approximated the nonconvex energy by quadratic
pieces as depicted in Fig.~\ref{fig:sar_experiments}. The label space
is chosen as $\Gamma = [0, 4 \pi]$ and the regularization parameter was
set to $\lambda = 0.005$.
Again, it is visible in Fig.~\ref{fig:sar_experiments} that the baseline method shows
label space discretization and  
fails to unwrap the depth map correctly if the number of labels is
chosen too low. The proposed method yields a smooth unwrapped
result using only $8$ labels.

\vspace{-1mm}
\subsection{Depth From Focus}
In depth from focus the task is to recover the depth of a
scene, given a stack of images each taken from a constant position
but in a different focal setting, so that in each image only the
objects of a certain depth are sharp. We achieve this by estimating the
depth of a point by locally maximizing its contrast over the set of
images. % (each corresponding to a certain depth value).
% Again, we assume the depth field to be smooth. 
% The input to our method is a 
% cost volume of size $480 \times 270 \times 373$, corresponding to
% $373$ differently focused color images.
We compute the cost by using the modified Laplacian function
\cite{Nayar-nakagawa-cvpr92} as a contrast measure.
Similar to the stereo experiments, we convexify the cost on each
label range by computing the convex hull. The results are shown
in Fig.~\ref{fig:dff_compare}. While the baseline method clearly
shows the label space discretization, the proposed approach 
yields a smooth depth map. 
Since the proposed method uses a convex lower bound of the lifted energy,
the regularizer has slightly more influence on the final result. 
This explains why the resulting depth maps in Fig.~\ref{fig:dff_compare}
and Fig.~\ref{fig:stereo_compare} look overall less noisy.

\section{Conclusion}

\begin{figure}
  \centering
  \captionsetup[subfloat]{labelformat=empty,justification=centering,singlelinecheck=false,margin=0pt}
  \subfloat[\scriptsize Piecewise convex energy]{
    \setlength\fheight{1.4cm} 
    \setlength\fwidth{3.5cm}
    \input{unwrapping_energy.tex} 
  }
  \subfloat[\scriptsize Input image]{
    \includegraphics[width=0.11\textwidth]{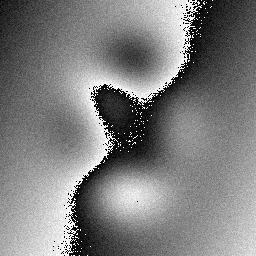} 
  }
  \subfloat[\scriptsize Ground truth]{
    \includegraphics[width=0.11\textwidth]{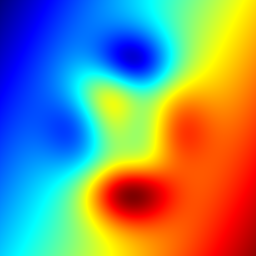} 
  }\\[-0.24cm]
  \subfloat[\scriptsize Baseline ($L=8$)]{
    \includegraphics[width=0.11\textwidth]{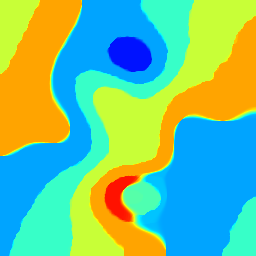} 
  }
  \subfloat[\scriptsize Baseline ($L=16$)]{
    \includegraphics[width=0.11\textwidth]{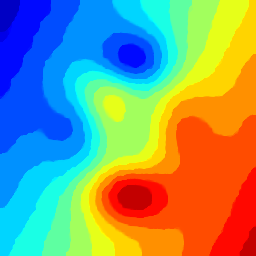} 
  }
  \subfloat[\scriptsize Baseline ($L=32$)]{
    \includegraphics[width=0.11\textwidth]{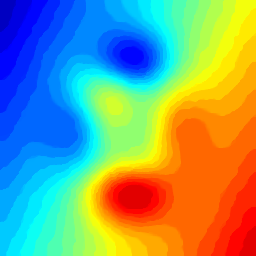} 
  }
  \subfloat[\scriptsize Proposed ($L=8$)]{
    \includegraphics[width=0.11\textwidth]{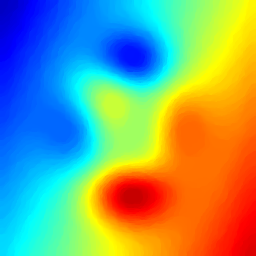} 
  }\\[1mm]
  \caption{We show the piecewise convex approximation of the phase unwrapping energy, followed by the
    cyclic input image and the unwrapped ground truth. With only 8 labels, the proposed method already
    yields a smooth reconstruction. The baseline method fails to unwrap
    the heightmap correctly using $8$ labels, and for $16$ and $32$
    labels, the discretization is still noticable.}
  \label{fig:sar_experiments}
\end{figure}

In this work we proposed a tight convex relaxation that can be interpreted as a sublabel--accurate formulation of classical multilabel problems. We showed that the local convex envelope involves infinitely many constraints, however we proved that it suffices to consider linearly many of those. The final formulation is a simple saddle-point problem that admits fast primal-dual optimization. Our method maintains sublabel accuracy even after discretization and for that reason outperforms existing spatially continuous methods. Interesting directions for future work include higher dimensional label spaces, manifold valued data and more general regularizers.

\begin{appendix}
\section{Appendix}
\begin{proof}[Proof of Proposition 1]
The proof follows from a direct calculation. We start with the definition of the biconjugate:
\begin{equation}
  \begin{aligned}
    \dat^{**}(\ul) &= \underset{\vl \in \bbR^k} \sup ~ \iprod{\ul}{\vl} - \left( \underset{1 \leq i \leq k} \min ~ \dat_i(\ul) \right)^* \\
    &= \underset{\vl \in \bbR^k} \sup ~ \iprod{\ul}{\vl} - \underset{1 \leq i \leq k} \max ~ \dat_i^*(\ul).
  \end{aligned}
  \label{eq:biconjugate}
\end{equation}
This shows the first equation inside the proposition. For the individual $\dat_i^*$ we again start with the definition of the convex conjugate:
\begin{equation}
  \begin{aligned}
    \dat_i^*(\vl) &= \underset{\alpha \in [0,1]} \sup ~ \iprod{\alpha \Gr_i + (1 - \alpha) \Gr_{i-1}}{\vl} - \\
    & \qquad \qquad \qquad \rho(\alpha \gamma_{i+1} + (1 - \alpha) \gamma_i) \\
    &= \underset{\alpha \in [0,1]} \sup ~ \iprod{\Gr_{i-1}}{\vl} + \alpha \vl_i - \rho(\gamma_i^\alpha ).
  \end{aligned}
  \label{eq:conjugate_2}
\end{equation}
Applying the substitution $\gamma_i^\alpha = \alpha \gamma_{i+1} + (1 - \alpha) \gamma_i$ and consequently $\alpha = \frac{\gamma_i^\alpha  - \gamma_i}{\gamma_{i+1} - \gamma_i}$ yields:
\begin{equation}
  \begin{aligned}
     &  \dat_i^*(\vl) = \underset{\gamma_i^\alpha  \in \Gamma_i} \sup ~ \iprod{\Gr_{i-1}}{\vl} + \frac{\gamma_i^\alpha  - \gamma_i}{\gamma_{i+1} - \gamma_i} \vl_i - \rho(\gamma_i^\alpha ) \\
   =& \iprod{\Gr_{i-1}}{\vl} - \frac{\gamma_i}{\gamma_{i+1} - \gamma_i} \vl_i + \underset{\gamma_i^\alpha  \in \Gamma_i} \sup ~  \gamma_i^\alpha  \frac{\vl_i}{\gamma_{i+1} - \gamma_i} - \rho(\gamma_i^\alpha ) \\
   =& \iprod{\Gr_{i-1}}{\vl} - \frac{\gamma_i}{\gamma_{i+1} - \gamma_i} \vl_i  + (\rho + \delta_{\Gamma_i})^*\left(\frac{\vl_i}{\gamma_{i+1} - \gamma_i}\right) \\
   =&:c_i(\vl) + \rho_i^*\left(\frac{\vl_i}{\gamma_{i+1} - \gamma_i}\right).
  \end{aligned}
  \label{eq:conjugate}
\end{equation}
\end{proof}

\begin{proof}[Proof of Proposition 2]
It is easy to see that 
$$ \dats^*(\vl) = \max_{i \in \{1, ..., L \}}\left( \sum_{l=1}^{i-1} \vl_l - \rho(\gamma_i)\right).$$
To compute the biconjugate, we write any input argument $\ul = \sum_{i=1}^k \mu_i \Gr_{i+1}$, and use $\dats^{**}=\dat^{**}$ to obtain
\begin{align*}
\dat^{**}(\ul) &= \sup_{\vl} \langle \ul, \vl \rangle - \max_{i \in \{1, ..., L \}}\left( \sum_{l=1}^{i-1} \vl_l - \rho(\gamma_i)\right) \\
&= \sup_{\vl} \sum_{i=1}^k \mu_i \sum_{l=1}^i \vl_l - \max_{i \in \{1, ..., L \}} \left( \sum_{l=1}^{i-1} \vl_l - \rho(\gamma_i)\right).
\end{align*}
Instead of taking the supremum of all $\vl$, we might as well take the supremum over all vectors $\textbf{p}$ with $\textbf{p}_i = \sum_{l=1}^i \vl_l$. Care has to be taken of the first summand in the second term of the above formulation. We obtain
\begin{align*}
&\sup_{\vl} \sum_{i=1}^k \mu_i \sum_{l=1}^i \vl_l - \max_{i \in \{1, ..., L \}} \left( \sum_{l=1}^{i-1} \vl_l - \rho(\gamma_i)\right), \\
=& \sup_{\textbf{p}} \sum_{i=1}^k \mu_i \textbf{p}_i - \max_{i \in \{2, ..., L \}} \max(\textbf{p}_{i-1} - \rho(\gamma_i),- \rho(\gamma_1)) , \\
=& \sup_{\textbf{p}} \sum_{i=1}^k \mu_i \textbf{p}_i - \max_{i \in \{1, ..., k \}} \max(\textbf{p}_{i} - \rho(\gamma_{i+1}),- \rho(\gamma_1)) , \\
=&\sum_{i=1}^k \mu_i \; \rho(\gamma_{i+1}) \\
 &+  \sup_{\textbf{p}} \sum_{i=1}^k \mu_i \textbf{p}_i - \max_{i \in \{1, ..., k \}} \max(\textbf{p}_{i},- \rho(\gamma_1)) , 
\end{align*}
Note that for any $\mu_i$ being negative, the supremum immediately yields infinity by taking $\textbf{p}_i \rightarrow - \infty$. Similarly, if $\sum_{i=1}^k \mu_i >1$ yields infinity by taking all $\textbf{p}_i \rightarrow \infty$. For $\mu_i \geq 0$ for all $i$, and $\sum_{i=1}^k \mu_i \leq 1$, we know that $ \sum_{i=1}^k \mu_i \textbf{p}_i \leq (\max_i \textbf{p}_i) \sum_{i=1}^k \mu_i$. Since equality can be obtained by choosing $\textbf{p}_l = \max_i \textbf{p}_i$ for all $l$, we can reduce the above supremum to 
\begin{align*}
&\sup_{z} \left(z \sum_{i=1}^k \mu_i  - \max(z,- \rho(\gamma_1)) \right) = \left(1 -\sum_{i=1}^k \mu_i\right) \rho(\gamma_1),
\end{align*}
where we used that the supremum over $z$ is attained at $z = -\rho(\gamma_1)$ (still assuming that $\sum_{i=1}^k \mu_i \leq 1$). Let us now undo our change of variable. It is easy to see that $\mu_k= \ul_k$, and $\mu_{i} = \ul_{i} - \ul_{i+1}$ for $i =1,...,k-1$. The latter leads to 
\begin{align*}
&\sum_{i=1}^k \mu_i \; \rho(\gamma_{i+1}) +  \left(1 -\sum_{i=1}^k \mu_i \right) \rho(\gamma_1) \\
&= \rho(\gamma_{k+1}) \ul_k + \sum_{i=1}^{k-1} (\ul_{i} - \ul_{i+1}) \; \rho(\gamma_{i+1}) +  (1 - \ul_1) \rho(\gamma_1) \\
&= \rho(\gamma_1) + \langle \ul, \textbf{r}\rangle,
\end{align*}
for $\textbf{r}_i = \rho(\gamma_{i+1})-\rho(\gamma_{i})$. Considering the aforementioned constraints of $\mu_i \geq 0$, and $\sum_{i=1}^k \mu_i\leq 1$, we finally find
\begin{align*}
\dat^{**}(\ul) = \begin{cases}
    \rho(\gamma_1) + \langle \ul, \textbf{r}\rangle &\text{if } ~ 1 \geq \ul_1 \geq ... \geq \ul_k \geq 0 ,\\
    \infty, & \text{else.}
\end{cases}
\end{align*}
\end{proof}

\begin{proof}[Proof of Proposition 3]
For the special case $k = 1$ the biconjugate from \eqref{eq:biconjugate} is just:
\begin{equation}
  \dat^{**}(\ul) = \underset{\vl \in \bbR} \sup ~ \ul \vl - \dat_1^*(\vl) = \dat_1^{**}(\ul).
\end{equation}
Now using the first line in \eqref{eq:conjugate}, $\dat_1^{**}$ becomes:
\begin{equation}
  \begin{aligned}
    \dat_1^{**}(\ul) &= \underset{\vl \in \bbR} \sup ~ \ul \vl -
    \underset{\gamma \in \Gamma} \sup ~ \frac{\gamma - \gamma_1}{\gamma_2 - \gamma_1}  \vl - \rho(\gamma)\\
    &= \underset{\vl \in \bbR} \sup ~ \vl \left(\ul + \frac{\gamma_1}{\gamma_2 - \gamma_1} \right) -
    \underset{\gamma \in \Gamma} \sup ~ \gamma \frac{\vl}{\gamma_2 - \gamma_1} - \rho(\gamma) \\
    &= \underset{\vl \in \bbR} \sup ~ \vl \left(\ul + \frac{\gamma_1}{\gamma_2 - \gamma_1}\right) -
    \rho^*\left(\frac{\vl}{\gamma_2 - \gamma_1} \right) \\
    &= \underset{\tilde \vl \in \bbR} \sup ~ \tilde \vl (\gamma_1 + \ul (\gamma_2 - \gamma_1)) -
    \rho^*(\tilde \vl) \\
    &= \rho^{**}(\gamma_1 + \ul (\gamma_2 - \gamma_1)),
  \end{aligned}
\end{equation}
where we used $\text{dom}(\rho) = \Gamma$ as well as the substitution $\vl = (\gamma_2 - \gamma_1) \tilde \vl$.
\end{proof}

\begin{proof}[Proof of Proposition 4]
  We compute the individual conjugate as:
  \begin{equation}
    \begin{aligned}
      &\reg_{i,j}^*(\ql) = \underset{\gl \in \bbR^{d \times k}}\sup \iprod{g}{\ql} - \reg_{i,j}(\ql)\\
      &= \underset{\alpha, \beta \in [0, 1]} \sup ~ \underset{\nu \in \bbR^d} \sup ~ \iprod{\ql}{(\Gr_i^{\alpha} - \Gr_j^{\beta}) \nu^{\mathsf{T}}} -  \left|\gamma_i^{\alpha}- \gamma_j^{\beta} \right| \normc{ \nu}_2\\
      &= \underset{\alpha, \beta \in [0, 1]} \sup ~ \underset{\nu \in \bbR^d} \sup ~ \iprod{\ql^{\mathsf{T}} (\Gr_i^{\alpha} - \Gr_j^{\beta})}{\nu} -  \left|\gamma_i^{\alpha}- \gamma_j^{\beta} \right| \normc{ \nu}_2\\
      &= \underset{\alpha, \beta \in [0, 1]} \sup ~ \underset{\nu \in \bbR^d} \sup ~ \iprod{\ql^{\mathsf{T}} (\Gr_i^{\alpha} - \Gr_j^{\beta})}{\nu} -  \left|\gamma_i^{\alpha}- \gamma_j^{\beta} \right| \normc{\nu}_2.
    \end{aligned}
  \end{equation}
  The inner maximum over $\nu$ is the conjugate of the $\ell_2$-norm scaled by $ \left|\gamma_i^{\alpha}- \gamma_j^{\beta} \right|$ evaluated at $\ql^{\mathsf{T}} \left(\Gr_i^{\alpha} - \Gr_j^{\beta}\right)$. This yields: 
  \begin{equation}
    \reg_{i,j}^*(\ql) = 
    \begin{cases}
      0, &\text{if }~ \left|\ql^{\mathsf{T}} \left(\Gr_i^{\alpha} -
        \Gr_j^{\beta}\right)\right|_2 \leq  \left|\gamma_i^{\alpha}- \gamma_j^{\beta} \right|,\\
      ~ & \hspace{3.5cm} \forall \alpha,\beta \in [0,1],\\
      \infty, &\text{else.}
    \end{cases}
  \end{equation}
  For the overall biconjugate we have:
  \begin{equation}
    \begin{aligned}
      \reg^{**}(\gl) &= \underset{\ql \in \bbR^{k \times d}} \sup ~
      \iprod{\ql}{\gl} - \underset{1 \leq i,j \leq k} \max ~
      \reg_{i,j}^*(\ql)\\
      &= \underset{\ql \in \mathcal{K}} \sup ~
      \iprod{\ql}{\gl}.
    \end{aligned}
  \end{equation}
  Since we have the $\max$ over all $1 \leq i,j \leq k$ conjugates,
  the set $\mathcal{K}$ is given as the intersection of the sets 
  described by the individual indicator functions $\reg_{i,j}$:
  \begin{equation}
    \begin{aligned}
      \mathcal{K} = \left\{ \ql \in \vphantom{\left| \ql^{\mathsf{T}} ( \Gr_i^{\alpha} - \Gr_j^{\beta} )\right|_2}\right.&\bbR^{k \times d} ~\mid~ \\ 
      & \left| \ql^{\mathsf{T}} ( \Gr_i^{\alpha} - \Gr_j^{\beta} )\right|_2 \leq \left| \gamma_i^{\alpha} - \gamma_j^{\beta}\right|, \\
      &  \left.\vphantom{\left| \ql^{\mathsf{T}} ( \Gr_i^{\alpha} - \Gr_j^{\beta} )\right|_2}\forall ~ 1 \leq i \leq j \leq k, ~ \forall \alpha, \beta \in [0, 1] \right\}.
    \end{aligned}
    \label{eq:constraints_infinite}
  \end{equation}
%  As $T$ is symmetric, this is clearly equivalent to:
%  \begin{equation}
%      \Phi^{**}(g) = \underset{q \in \mathcal{K}} \sup ~
%      \iprod{q}{g},
%  \end{equation}
%  for a different constraint set $\mathcal{K}$ given as:
%  \begin{equation}
%    \begin{aligned}
%      \mathcal{K} = \bigl \{ q \in &\bbR^{k \times d} ~\mid~ \\
%      & |q^{\mathsf{T}} ( \Gr_i^{\alpha} - \Gr_j^{\beta} )|_2 \leq d(\gamma_i^{\alpha}, \gamma_j^{\beta}), \\
%      & \forall ~ 1 \leq i, j \leq k, ~ \forall \alpha, \beta \in [0, 1] \bigr \}.
%    \end{aligned}
%    \label{eq:constraints_infinite_simple}
%  \end{equation}
\end{proof}

\begin{proof}[Proof of Proposition 5]
%  Let $d(\gamma_i^{\alpha}, \gamma_j^{\beta}) = |\gamma_i^{\alpha} - \gamma_j^{\beta}|$. 
  First we rewrite \eqref{eq:constraints_infinite} by expanding the matrix-vector product 
  into sums:
  \begin{equation}
    \begin{aligned}
      &\left|\sum_{l=j}^{i-1} \ql_l + \alpha \ql_i - \beta \ql_j\right|_2 \leq \left| \gamma_i^{\alpha} - \gamma_j^{\beta}\right|,\\
      &\forall ~ 1 \leq j \leq i \leq k, ~ \forall \alpha, \beta \in [0, 1].
    \end{aligned}
    \label{eq:constraints_infinite_reduced}
  \end{equation}
  Since the other cases for $1 \leq i \leq j \leq k$ in \eqref{eq:constraints_infinite} are equivalent to \eqref{eq:constraints_infinite_reduced}, it is enough to consider \eqref{eq:constraints_infinite_reduced} instead of \eqref{eq:constraints_infinite}.

  Let $\gamma_1 < \gamma_2 < \hdots < \gamma_L$. We show the equivalences:
  \begin{center}
    \eqref{eq:constraints_infinite_reduced}
  \end{center}
  \begin{center}
    $\Leftrightarrow$
  \end{center}
  \begin{equation}
    \begin{aligned}
      &\left| \sum_{l=j}^{i} \ql_l \right|_2 \leq \gamma_{i+1} - \gamma_j,
      ~\forall ~ 1 \leq j \leq i \leq k.
    \end{aligned}
    \label{eq:constraints_quadratic}
  \end{equation}
  \begin{center}
    $\Leftrightarrow$
  \end{center}
  \begin{equation}
    \normc{\ql_i}_2 \leq \gamma_{i+1} - \gamma_i, ~ \forall ~ 1 \leq i \leq k.
    \label{eq:constraints_linear}
  \end{equation}
  The direction ``\eqref{eq:constraints_infinite_reduced} $\Rightarrow$ \eqref{eq:constraints_quadratic}'' follows by setting $\alpha = 1$ and $\beta = 0$ in \eqref{eq:constraints_infinite_reduced}, and ``\eqref{eq:constraints_quadratic} $\Rightarrow$ \eqref{eq:constraints_linear}'' follows by setting $i = j$ in \eqref{eq:constraints_quadratic}.

The direction ``\eqref{eq:constraints_linear} $\Rightarrow$ \eqref{eq:constraints_quadratic}'' can be proven by a quick calculation:
\begin{equation}
  \left| \sum_{l=j}^{i} \ql_l \right|_2 \leq \sum_{l=j}^i \normc{\ql_l}_2 \leq \sum_{l=j}^i \gamma_{l+1} - \gamma_l = \gamma_{i+1} - \gamma_j.
\end{equation}
It remains to show ``\eqref{eq:constraints_quadratic} $\Rightarrow$ \eqref{eq:constraints_infinite_reduced}''. We start with the case $j = i$:
\begin{equation}
  \begin{aligned}
    \normc{\alpha \ql_i - \beta \ql_i}_2 &= |\alpha - \beta| \normc{\ql_i}_2 \\
    &\leq |\alpha - \beta| (\gamma_{i+1} - \gamma_i)\\
    &= | (\gamma_{i+1} - \gamma_i) \alpha -  (\gamma_{i+1} - \gamma_i) \beta|\\
    &= | (\alpha - \beta) (\gamma_{i+1} - \gamma_i) | = |\gamma_i^{\alpha} - \gamma_i^{\beta}|.
  \end{aligned}
\end{equation}
Now let $j < i$. Since $\gamma_j < \gamma_i$ it also holds that $\gamma_j^{\beta} \leq \gamma_i^{\alpha}$, thus
it is equivalent to show \eqref{eq:constraints_infinite_reduced} without the absolute value on the right hand side.

First we show that ``\eqref{eq:constraints_quadratic} $\Rightarrow$ \eqref{eq:constraints_infinite_reduced}'' for $\beta \in \{ 0, 1\}$ and $\alpha \in [0,1]$:
\begin{equation}
  \begin{aligned}
    &\left|\sum_{l=j+1}^{i-1} \ql_l + \alpha \ql_i + (1 - \beta) \ql_j\right|_2 \\
    &\leq \left|\sum_{l=j+1}^{i-1} \ql_l + (1 - \beta) \ql_j\right|_2 + \alpha \normc{\ql_i}_2 \\
    &\overset{\text{for } \beta=0 \text{ or } \beta=1} \leq \gamma_i - \gamma_j^{\beta} + \alpha(\gamma_{i+1} - \gamma_i)\\
    &= \gamma_i^{\alpha} - \gamma_j^{\beta}.
  \end{aligned}
  \label{eq:temp}
\end{equation}
Using a similar argument we show that, using the above, ``\eqref{eq:constraints_quadratic} $\Rightarrow$ \eqref{eq:constraints_infinite_reduced}'' for all $\alpha, \beta \in [0,1]$.
\begin{equation}
  \begin{aligned}
    &\left|\sum_{l=j+1}^{i-1} \ql_l + \alpha \ql_i + (1 - \beta) \ql_j\right|_2 \\
    &\leq \left|\sum_{l=j+1}^{i-1} \ql_l + \alpha \ql_i\right|_2 + (1 - \beta) \normc{\ql_j}_2\\
    &\overset{\text{using } \eqref{eq:temp}, \beta=1} \leq \gamma_i^{\alpha} - \gamma_{j+1} + (1 - \beta) (\gamma_{j+1} - \gamma_j)\\
    &= \gamma_i^{\alpha} - \gamma_j^{\beta}.
  \end{aligned}
\end{equation}

\end{proof}

\end{appendix}

{\small
  \bibliographystyle{ieee}
  \bibliography{references}
}

\end{document}